\documentclass[sn-mathphys]{sn-jnl}

\jyear{2021}%

\usepackage{amsmath,amsfonts,bm}

\def\eqref#1{equation~\ref{#1}}

\def\1{\bm{1}}

\def\vf{{\bm{f}}}
\def\vg{{\bm{g}}}

\def\vi{{\bm{i}}}
\def\vj{{\bm{j}}}

\def\vy{{\bm{y}}}

\def\mA{{\bm{A}}}
\def\mB{{\bm{B}}}
\def\mC{{\bm{C}}}
\def\mD{{\bm{D}}}

\def\mG{{\bm{G}}}
\def\mH{{\bm{H}}}
\def\mI{{\bm{I}}}

\def\mL{{\bm{L}}}

\def\mO{{\bm{O}}}

\def\mU{{\bm{U}}}

\def\mW{{\bm{W}}}
\def\mX{{\bm{X}}}
\def\mY{{\bm{Y}}}
\def\mZ{{\bm{Z}}}

\def\mLambda{{\bm{\Lambda}}}

\DeclareMathAlphabet{\mathsfit}{\encodingdefault}{\sfdefault}{m}{sl}
\SetMathAlphabet{\mathsfit}{bold}{\encodingdefault}{\sfdefault}{bx}{n}

\def\sA{{\mathbb{A}}}
\def\sB{{\mathbb{B}}}

\def\sH{{\mathbb{H}}}
\def\sI{{\mathbb{I}}}

\def\sO{{\mathbb{O}}}

\def\sS{{\mathbb{S}}}
\def\sT{{\mathbb{T}}}
\def\sU{{\mathbb{U}}}
\def\sV{{\mathbb{V}}}
\def\sW{{\mathbb{W}}}
\def\sX{{\mathbb{X}}}

\usepackage{amsmath}

\usepackage{natbib}
\usepackage{MnSymbol}
\usepackage{tikz}
\usepackage{program}
\usepackage[all]{xy}
\usepackage[arrowdel]{physics}
\usepackage{amsmath}
\usepackage{amsfonts}
\usepackage{amssymb}
\usepackage{tikz}
\usepackage{color}
\usepackage{nicematrix}

\theoremstyle{thmstyleone}%
\newtheorem{lemma}{Lemma}
\theoremstyle{thmstyletwo}%

\theoremstyle{thmstylethree}%
\newtheorem{definition}{Definition}%

\begin{document}

\title[Hy, Khang and Kondor]{Learning to Solve Multiresolution Matrix Factorization by Manifold Optimization and Evolutionary Metaheuristics} 


\author*[1]{\fnm{Truong Son} \sur{Hy}}\email{TruongSon.Hy@indstate.edu}

\author[1]{\fnm{Thieu} \sur{Khang}}\email{thieukhang.ng@gmail.com}

\author[2]{\fnm{Risi} \sur{Kondor}}\email{risi@uchicago.edu}


\affil[1]{\orgdiv{Department of Mathematics and Computer Science}, \orgname{Indiana State University}, \orgaddress{\street{200 N. 7th St.}, \city{Terre Haute}, \postcode{47809}, \state{IN}, \country{United States}}}

\affil[2]{\orgdiv{Department of Computer Science}, \orgname{University of Chicago}, \orgaddress{\street{5730 South Ellis Ave.}, \city{Chicago}, \postcode{60637}, \state{Illinois}, \country{United States}}}




\abstract{Multiresolution Matrix Factorization (MMF) is unusual amongst fast matrix factorization algorithms in that it does not make a low rank assumption. This makes MMF especially well suited to modeling certain types of graphs with complex multiscale or hierarchical strucutre. While MMF promises to yields a useful wavelet basis, finding the factorization itself is hard, and existing greedy methods tend to be brittle. In this paper, we propose a ``learnable'' version of MMF that carfully optimizes the factorization using metaheuristics, specifically evolutionary algorithms and directed evolution, along with Stiefel manifold optimization through backpropagating errors. We show that the resulting wavelet basis far outperforms prior MMF algorithms and gives comparable performance on standard learning tasks on graphs. Furthermore, we construct the wavelet neural networks (WNNs) learning graphs on the spectral domain with the wavelet basis produced by our MMF learning algorithm. Our wavelet networks are competitive against other state-of-the-art methods in molecular graphs classification and node classification on citation graphs. We release our implementation at \url{https://github.com/HySonLab/LearnMMF}.
}




\maketitle

\section{Introduction} \label{sec:Introduction}

Graph convolutional networks (GCNs) have become a powerful tool for learning from graph-structured data, which appear in various fields such as social networks, molecular chemistry, and recommendation systems. Unlike traditional data represented in grids or sequences, graphs have complex, irregular structures with nodes connected by edges, making conventional convolutional operations unsuitable.

To tackle this challenge, researchers have adapted convolution to the graph domain. One approach uses the Graph Fourier transform (GFT) \cite{ae482107de73461787258f805cf8f4ed}, which relies on the eigendecomposition of the graph Laplacian matrix. The GFT represents a graph signal in terms of its frequency components, similar to classical signal processing.

The graph convolution operator in the spectral domain is defined as:
\[
\vf *_{\mathcal{G}} \vg = \mU \big( (\mU^T \vg) \odot (\mU^T \vf) \big),
\]
where $\vf$ is the graph signal, $\vg$ is the convolution kernel, $\mU$ are the eigenvectors of the graph Laplacian, and $\odot$ denotes the element-wise Hadamard product. This operation simplifies to matrix multiplication, making it computationally efficient.

However, the GFT approach has significant limitations. First, computing the eigendecomposition is often infeasible for large graphs due to its high computational cost. Second, the learned filters are not localized in the vertex domain, making it difficult to capture local structures effectively.

These limitations underscore the need for alternative methods that efficiently perform convolution on graphs while preserving their local and global properties. To address these issues, we propose a modified spectral graph network based on the Multiresolution Matrix Factorization (MMF) \cite{pmlr-v32-kondor14} wavelet basis instead of the Laplacian eigenbasis. This approach offers several advantages:
(i) the wavelets are generally localized in both vertex and frequency domains,
(ii) the individual basis transforms are sparse, and
(iii) MMF provides an efficient way to decompose graph signals into components at different levels of granularity, offering an excellent basis for sparse approximations.

In many machine learning problems, large matrices have complex hierarchical structures that traditional low-rank methods struggle to capture. MMF is an alternative paradigm designed to capture structure at multiple scales. It is particularly effective for compressing the adjacency or Laplacian matrices of complex graphs, such as social networks \cite{pmlr-v32-kondor14}. MMF factorizations have a number of advantages, including the fact that they are easy to invert and have an 
interpretation as a form of wavelet analysis on the matrix and consequently on the underlying graph.  
The wavelets can be used for finding sparse approximations of graph signals. 

Finding the actual MMF factorization, however, is a hard optimization problem combining elements of 
continuous and combinatorial optimization. 
Most of the existing MMF algorithms just tackle this with a variety of greedy heuristics and are consequently 
brittle: the resulting factorizations typically have large variance and most of the time yield factorizations 
that are far from the optimal \cite{pmlr-v51-teneva16,8099564,ding2018multiresolution, pmlr-v196-hy22a}. 

This paper proposes an alternative approach to MMF optimization. Specifically, we use an iterative method that optimizes the factorization by backpropagating the factorization error and applying metaheuristic strategies to solve the combinatorial aspects. Although more computationally intensive than greedy methods, this ``learnable'' MMF produces higher quality factorizations and a wavelet basis that better reflects the structure of the underlying matrix or graph. Consequently, this leads to improved performance in downstream tasks.

To demonstrate the effectiveness of our learnable MMF algorithm, we introduce a wavelet extension of the Spectral Graph Networks algorithm \cite{ae482107de73461787258f805cf8f4ed}, called the Wavelet Neural Network (WNN). Our experiments show that combining learnable MMF with WNNs achieves state-of-the-art results on several graph learning tasks. By addressing the inefficiencies of the approaches based on eigendecomposition, our method provides a fast and effective convolution operation on graphs. Beyond benchmark performance, the enhanced stability of MMF optimization and the hierarchical structure's similarity to deep neural networks suggest that MMF could be integrated with other learning algorithms in the future.

\section{Related work} \label{sec:Related-work}

\textbf{Multiresolution matrix factorization.} Compressing and estimating large matrices has been extensively studied from various directions,
including (i) column/row selection methods \cite{10.1137/S0097539704442684, Drineas2006FastMC, 10.1137/S0097539704442702, 10.1145/1219092.1219097, doi:10.1137/090771806}, 
(ii) Nystr\"{o}m Method 
\cite{NIPS2000_19de10ad, pmlr-v5-kumar09a, JMLR:v13:kumar12a}, 
(iii) randomized linear algebra 
\cite{10.1561/2200000035}, and (iv) sparse PCA 
\cite{pmlr-v9-jenatton10a}.
Many of these methods come with explicit guarantees but typically make the assumption that 
the matrix to be approximated is low rank. 
MMF is more closely related to other works on constructing wavelet bases on discrete spaces, 
including wavelets defined based on diagonalizing the diffusion operator or the normalized graph Laplacian 
\cite{COIFMAN200653, HAMMOND2011129} and multiresolution on trees 
\cite{10.5555/3104322.3104370, 10.1214/07-AOAS137}. 
MMF has been used for matrix compression \cite{pmlr-v51-teneva16, pmlr-v196-hy22a}, 
kernel approximation \cite{ding2018multiresolution} 
and inferring semantic relationships in medical imaging data \cite{8099564}. 

\cite{pmlr-v32-kondor14} proposed a greedy method for multiresolution matrix factorization, which outperforms Nystr\"{o}m methods on matrices with a multilevel structure. Other approaches to solving MMF include utilizing parallelism \cite{pmlr-v51-teneva16} and implementing an incremental updating scheme \cite{8099564}. However, these methods rely on suboptimal localized heuristics, whereas our learning algorithm directly addresses global optimization.

\noindent\textbf{Graph neural networks.} Graph neural networks (GNNs) utilizing the generalization of convolution concept to graphs have been popularly applied to many learning tasks such as estimating quantum chemical computation \cite{HyEtAl2018, pmlr-v70-gilmer17a}, modeling physical systems \cite{NIPS2016_3147da8a}, predicting the progress of an epidemic or pandemic \cite{pmlr-v184-hy22a, nguyen2023predicting}, etc. 

Spectral methods such as \cite{ae482107de73461787258f805cf8f4ed} provide one way to define convolution on graphs via convolution theorem and Graph Fourier transform (GFT). 
\cite{DBLP:journals/cgf/BoscainiMMBCV15} and \cite{10.1145/3137609} both propose methods for learning class-specific descriptors for deformable shapes. Boscaini's approach \cite{DBLP:journals/cgf/BoscainiMMBCV15} uses localized spectral convolutional networks, while Huang's method \cite{10.1145/3137609} involves training a network to embed similar points close to each other in descriptor space. \cite{Defferrard16} introduced a formulation of CNNs in the context of spectral graph theory, enabling the design of fast localized convolutional filters on graphs. \cite{hammond:inria-00541855} proposed a method for constructing wavelet transforms of functions on weighted graphs using spectral graph theory, defining scaling through the spectral decomposition of the discrete graph Laplacian.
To address the high computational cost of GFT, \cite{xu2018graph} proposed to use the diffusion wavelet bases as previously defined by \cite{COIFMAN200653} instead for a faster transformation.

\section{Background on Multiresolution Matrix Factorization}


The \textit{Multiresolution Matrix Factorization} (MMF) of a matrix $\mA\in\mathbb{R}^{n\times n}$ is a factorization of the form
\begin{equation}
\mA = \mU_1^T \mU_2^T \dots \mU_L^T \mH \mU_L \dots \mU_2 \mU_1
\label{eq::mmf}
\end{equation}
where the $\mH$ and $\mU_1, \dots, \mU_L$ matrices conform to the following constraints: 
\begin{itemize}

\item Each $\mathbf{U}_\ell$ is an orthogonal matrix representing a $k$-point rotation for some small $k$, meaning it only rotates $k$ coordinates at a time. These matrices are essentially identity matrices with non-zero entries at a small set of coordinates. For a formal definition of these matrices, please refer to Def. \ref{def:rotation-matrix}.

\item We define \([n] = \{1, 2, 3, \ldots, n\}\) and \(\mathbb{I}_{\ell}\) as the set of \(k\) coordinates rotated by \(\mU_\ell\). There is a nested sequence of sets $\sS_L \subseteq \cdots \subseteq \sS_1 \subseteq \sS_0 = [n]$ such that $\mathbb{I}_{\ell} \subseteq \sS_\ell$.

\item $\mH$ is an $\sS_L$-core-diagonal matrix that is diagonal with an additional small 
$\sS_L\times \sS_L$ dimensional ``core'' at specific coordinates in $\sS_L$. The remaining entries are the same as those in a diagonal matrix. A formal definition of $\sS_L$-core-diagonal is at Def. \ref{def:core-diagonal}.
\end{itemize}

{

    \(\sS_{\ell-1}\) is can be viewed as the ``active set'' at the \(\ell^{\text{th}}\) level because \(\mathbf{U}_\ell\) is identity matrix outside the set \([n] \setminus \sS_{\ell-1}\). The \(\sS\) sets form a nested sequence indicating that when \(\mathbf{U}_\ell\) is applied at a particular level, the elements in \(\sS_{\ell} \setminus \sS_{\ell-1}\) are excluded from the active set and are not processed in future steps. This process of reducing the active set continues through all \(L\) levels, resulting in a nested subspace interpretation for the sequence of transformation. \cite{pmlr-v32-kondor14} makes the connection between MMF and multiresolution analysis \cite{192463}. 

    This multiresolution factorization reveals structure at multiple scales by sequentially applying sparse orthogonal transforms to \(A\). Each transform affects only a small set of coordinates \(\mathbb{I}_{\ell}\) in \(A\), leaving the rest unchanged. Initially, an orthogonal transform is applied, and the subset of rows and columns of \(U_1 A U_1^T\) that interact the least with the rest of the matrix capture the finest scale structure of \(A\). These corresponding rows of \(U_1\) are labeled as level one wavelets and remain invariant in subsequent steps. The process continues with a second orthogonal transform to produce \(U_2 U_1 A U_2^T U_1^T\), and this pattern is repeated, resulting in an \(L\)-level factorization as shown in Eq. \ref{eq::mmf}. The sequence of matrices, \(U_1 A U_1^T\), \(U_2 U_1 A U_2^T U_1^T\), \(\dots\) \(,H\) can be interpreted as compressed versions of \(A\) \cite{pmlr-v32-kondor14}.
}

Finding the best MMF factorization to a symmetric matrix $\mA$ involves solving
\begin{equation}
\min_{\substack{\sS_L \subseteq \cdots \subseteq \sS_1 \subseteq \sS_0 = [n] \\ \mH \in \sH^{\sS_L}_n; \mU_1, \dots, \mU_L \in \sO}} \| \mA - \mU_1^T \dots \mU_L^T \mH \mU_L \dots \mU_1 \|.
\label{eq:mmf-opt}
\end{equation}
Assuming that we measure error in the Frobenius norm, (\ref{eq:mmf-opt}) is equivalent to
\begin{equation}
\min_{\substack{\sS_L \subseteq \cdots \subseteq \sS_1 \subseteq \sS_0 = [n] \\ \mU_1, \dots, \mU_L \in \sO}} \| \mU_L \dots \mU_1 \mA \mU_1^T \dots \mU_L^T \|^2_{\text{resi}},
\label{eq:mmf-resi}
\end{equation}
where $\| \cdot \|_{\text{resi}}^2$ is the squared residual norm 
$\|\mH \|_{\text{resi}}^2 = \sum_{i \neq j; (i, j) \not\in \sS_L \times \sS_L} \lvert \mH_{i, j} \rvert^2$. 

There are two fundamental difficulties in MMF optimization: 
finding the optimal nested sequence of $\sS_\ell$ is a combinatorially hard 
(e.g., there are ${d_\ell \choose k}$ ways to choose $k$ indices out of $\sS_\ell$); 
and the solution for $\mU_\ell$ must satisfy the orthogonality constraint such that 
$\mU_\ell^T \mU_\ell = \mI$. 
The existing literature on solving this optimization problem 
\cite{pmlr-v32-kondor14, pmlr-v51-teneva16, 8099564, ding2018multiresolution} has various heuristic elements and has a number of limitations. 
First of all, there is no guarantee that the greedy heuristics (e.g., clustering) used in selecting $k$ rows/columns $\sI_\ell = \{i_1, .., i_k\} \subset \sS_\ell$ for each rotation return a globally optimal factorization. 
Instead of direct optimization for each rotation $\mU_\ell \triangleq \mI_{n - k} \oplus_{\sI_\ell} \mO_\ell$ where $\mO_\ell \in \sS\sO(k)$ 
globally and simultaneously with the objective (\ref{eq:mmf-opt}), Jacobi MMFs (see Proposition 2 of \cite{pmlr-v32-kondor14}) apply the greedy strategy of optimizing 
them locally and sequentially.
Again, this does not necessarily lead to a \textit{globally} optimal combination of rotations. Furthermore, most MMF algorithms are limited to the simplest case of $k = 2$ where 
$\mU_\ell$ is just a Givens rotation, which can be parameterized by a single variable, the rotation angle $\theta_\ell$.
This makes it possible to optimize the greedy objective by simple gradient descent, but 
larger rotations would yield more expressive factorizations and better approximations.

In contrast, we propose an iterative algorithm to directly optimize the global MMF objective 
(\ref{eq:mmf-opt}): 
\begin{itemize}
\item We use gradient descent algorithm on the Stiefel manifold to optimize all rotations 
$\{\mU_\ell\}_{\ell = 1}^L$ \textit{simultaneously}, whilst satisfying the orthogonality constraints. 
Importantly, the Stiefel manifold optimization is not limited to $k = 2$ case 
(Section ~\ref{sec:stiefel}).
\item We try to solve the problem of finding the optimal nested sequence $\sS_L \subseteq \cdots \subseteq \sS_1 \subseteq \sS_0 = [n]$ with metaheuristics like evolutionary algorithm and directed evolution. The cost function for this optimization problem is the value returned by the Stiefel
manifold optimization algorithm in equation (2).
\end{itemize}
We show that the resulting learning-based MMF algorithm outperforms existing greedy MMFs and other 
traditional baselines for matrix approximation in various scenarios (see Section \ref{sec:Experiments}).

Our mathematical notations are detailed in Appendix \ref{sec:Notation}. More background of MMF is included in Appendix \ref{sec:MMF}.


\section{Stiefel Manifold Optimization} \label{sec:stiefel}

The MMF optimization problem in (\ref{eq:mmf-opt}) and (\ref{eq:mmf-resi}) is equivalent to
\begin{equation}
\min_{\sS_L \subseteq \cdots \subseteq \sS_1 \subseteq \sS_0 = [n]} \min_{\mU_1, \dots, \mU_L \in \sO} \| \mU_L \dots \mU_1 \mA \mU_1^T \dots \mU_L^T \|^2_{\text{resi}}.
\label{eq:mmf-two-phases}
\end{equation}

In order to solve the inner optimization problem of (\ref{eq:mmf-two-phases}), 
we consider the following generic optimization with orthogonality constraints \cite{edelman1998geometry}:
\begin{equation}
\min_{\mX \in \mathbb{R}^{n \times p}} \mathcal{F}(\mX), \ \ \text{s.t.} \ \ \mX^T \mX = \mI_p,
\label{eq:opt-prob}
\end{equation}
where $\mI_p$ is the identity matrix and $\mathcal{F}(\mX): \mathbb{R}^{n \times p} \rightarrow \mathbb{R}$
is a differentiable function. 
The feasible set $\mathcal{V}_p(\mathbb{R}^n) = \{\mX \in \mathbb{R}^{n \times p}: \mX^T \mX = \mI_p\}$ is referred to as the 
Stiefel manifold of $p$ orthonormal vectors in $\mathbb{R}^{n}$ that has dimension equal to $np - \frac{1}{2}p(p + 1)$. 
We will view $\mathcal{V}_p(\mathbb{R}^n)$ as an embedded submanifold of $\mathbb{R}^{n \times p}$. 

When there is more than one orthogonal constraint, (\ref{eq:opt-prob}) is written as
\begin{equation}
\min_{\mX_1 \in \mathcal{V}_{p_1}(\mathbb{R}^{n_1}), \dots, \mX_q \in \mathcal{V}_{p_q}(\mathbb{R}^{n_q})} \mathcal{F}(\mX_1, \dots, \mX_q)
\label{eq:opt-prob-extended}
\end{equation}
where there are $q$ variables with corresponding $q$ orthogonal constraints. 

For example, in the MMF optimization problem (\ref{eq:mmf-opt}), suppose we are already given $\sS_L \subseteq \cdots \subseteq \sS_1 \subseteq \sS_0 = [n]$ meaning that the indices of active rows/columns at each resolution were already determined, for simplicity. In this case, we have $q = L$ number of variables such that each variable $\mX_\ell = \mO_\ell \in \mathbb{R}^{k \times k}$, where $\mU_\ell = \mI_{n - k} \oplus_{\sI_\ell} \mO_\ell \in \mathbb{R}^{n \times n}$ in which $\sI_\ell$ is a subset of $k$ indices from $\sS_\ell$, must satisfy the orthogonality constraint. The corresponding objective function is 
\begin{equation}
\mathcal{F}(\mO_1, \dots, \mO_L) = \| \mU_L \dots \mU_1 \mA \mU_1^T \dots \mU_L^T \|^2_{\text{resi}}.
\label{eq:mmf-core}
\end{equation}
Details about Stiefel manifold optimization are included in Appendix \ref{sec:proof}.

\section{Meta-heuristics}

\subsection{Problem Formulation} \label{sec:ea-problem}

{
We frame the task of identifying the optimal nested sequence of sets 
$\sS_L \subseteq \cdots \subseteq \sS_1 \subseteq \sS_0 = [n]$ as learning a set of wavelet indices to solve 
the MMF optimization in (\ref{eq:mmf-opt}). This involves two primary components for index selection at each resolution level $\ell \in \{1, \ldots, L\}$:

\begin{itemize}
\item Select the set of indices $\sT_\ell \subset \sS_{\ell - 1}$ representing the rows/columns to be wavelets at this level, which are then eliminated by defining $\sS_\ell = \sS_{\ell - 1} \setminus \sT_\ell$. To simplify computation, we assume that each resolution level selects only one row/column as the wavelet, such that $\lvert \sT_\ell \rvert = 1$. Consequently, the cardinality of $\sS_\ell$ decreases by 1 at each level, giving $d_\ell = n - \ell$. The core block size of $\mH$ becomes $(n - L) \times (n - L)$, corresponding to exactly $n - L$ active rows/columns at the end.
\item Select $k - 1$ indices $\sI_\ell = \{i_1, \ldots, i_{k - 1}\} \subset \sS_{\ell - 1}$ to construct the corresponding rotation matrix $\mU_\ell$ (see Section ~\ref{sec:stiefel}).
\end{itemize}

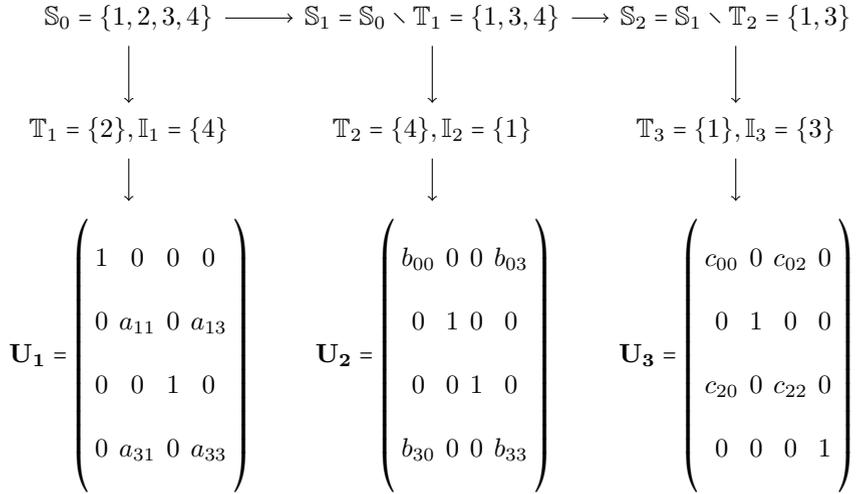
\begin{figure}[t]
    \centering
    \begin{tikzpicture}[node distance=1cm, auto]

        \node (S0) at (0, 0) {\(\sS_0 = \{1, 2, 3, 4\}\)};
        \node (S1) at (4, 0) {\(\sS_1 = \sS_0 \setminus \sT_1 = \{1, 3, 4\}\)};
        \node (S2) at (8, 0) {\(\sS_2 = \sS_1 \setminus \sT_2 = \{1, 3\}\)};

        \node (IT1) at (0, -1.5) {\(\sT_1 = \{2\}, \sI_1 = \{4\}\)};
        \node (IT2) at (4, -1.5) {\(\sT_2 = \{4\}, \sI_2 = \{1\}\)};
        \node (IT3) at (8, -1.5) {\(\sT_3 = \{1\}, \sI_3 = \{3\}\)};

        \node (U1) at (0, -4.5) {
            \(\mathbf{U_1} = \begin{pNiceMatrix}[margin]
                1 & 0 & 0 & 0 \\
                0 & a_{11} & 0 & a_{13} \\
                0 & 0 & 1 & 0 \\
                0 & a_{31} & 0 & a_{33} 
            \end{pNiceMatrix}\)
        };
        \node (U2) at (4, -4.5) {
            \(\mathbf{U_2} = \begin{pNiceMatrix}[margin]
                b_{00} & 0 & 0 & b_{03} \\
                0 & 1 & 0 & 0 \\
                0 & 0 & 1 & 0 \\
                b_{30} & 0 & 0 & b_{33} 
            \end{pNiceMatrix}\)
        };
        \node (U3) at (8, -4.5) {
            \(\mathbf{U_3} = \begin{pNiceMatrix}[margin]
                c_{00} & 0 & c_{02} & 0 \\
                0 & 1 & 0 & 0 \\
                c_{20} & 0 & c_{22} & 0 \\
                0 & 0 & 0 & 1 
            \end{pNiceMatrix}\)
        };

        \draw[->, shorten >=1pt, shorten <=1pt] (S0) -- (S1);
        \draw[->, shorten >=1pt, shorten <=1pt] (S1) -- (S2);

        \draw[->, shorten >=1pt, shorten <=1pt] (S0) -- (IT1);
        \draw[->, shorten >=1pt, shorten <=1pt] (S1) -- (IT2);
        \draw[->, shorten >=1pt, shorten <=1pt] (S2) -- (IT3);

        \draw[->, shorten >=1pt, shorten <=1pt] (IT1) -- (U1);
        \draw[->, shorten >=1pt, shorten <=1pt] (IT2) -- (U2);
        \draw[->, shorten >=1pt, shorten <=1pt] (IT3) -- (U3);

    \end{tikzpicture}
    \caption{Visualization of the nested set selection process for a \(4 \times 4\) matrix \(\mathbf{A}\) with \(L = 3\) and \(k = 2\). The process, depicted from left to right, demonstrates the trimming of the set \(\sS\). The sets \(\sT_\ell\) and \(\sI_\ell\) are chosen by metaheuristics, while the orthogonal transform \(\sU_\ell\) rotates all \(k\) coordinates in \(\sT_\ell \cup \sI_\ell\).}
    \label{fig:trimming_sets}
\end{figure}

A small example to illustrate this index selection is given in Fig. \ref{fig:trimming_sets}. The two algorithms use the Frobenius error from the Stiefel manifold optimization as the fitness function. In the second step of selecting \(k - 1\) indices, we identify the \(k - 1\) indices whose rows are closest to the wavelet row in Euclidean distance (see \ref{sec:Optimization} for the inspiration of this heuristic), reducing this problem to finding an ordered set of \(L\) wavelet indices.

In both metaheuristics, the candidate solution is an ordered set of \(L\) indices chosen as wavelet indices. The fitness function employed is the initial cost from the Stiefel manifold optimization algorithm, without training iterations, to estimate solution quality. This approach is designed to minimize the computational time when running the metaheuristics, as the optimization phase of the Stiefel manifold algorithm is only executed after identifying the best solution through the metaheuristics. Thus, the initial cost from the Stiefel manifold optimization serves as a cost-effective estimate of solution quality.

}

\subsection{Evolutionary Algorithm}

{
We employ a metaheuristics-based approach grounded in evolutionary algorithms to solve the optimization problem of finding the optimal nested sequence for Multiresolution Matrix Factorization (MMF). 

Evolutionary algorithms \cite{MUHLENBEIN198865}, inspired by the process of natural selection and genetics, are particularly effective for complex optimization tasks. Our method iteratively improves a population of candidate solutions by applying operations such as selection, crossover, and mutation. The selection process identifies the most promising candidates based on a fitness function, while crossover and mutation introduce genetic diversity, enabling the exploration of the solution space.

\begin{algorithm}[H]
\caption{Evolutionary Algorithm (EA) for MMF} \label{alg:ea}
\begin{algorithmic}[1]
\State \textbf{Input:} Matrix \( \mA \) to factorize, number of resolution levels \( L \), number of indices chosen for each resolution level \( k \), size of the matrix \( n \), maximum population size \( p_{\max} \) (must be even), number of iterations \( i_{\max} \), mutation rate \( m \) (\( 0 \leq m \leq 1 \)), the fitness function \( f \) representing the Frobenius error from the Stiefel manifold optimization algorithm.

\State Initialize the population \( P \) with \( p_{\max} \) random ordered sets of size \( L \) from the range \([1, n]\)
\State Initialize \( \sigma^* \) as a random solution
\For{$i = 1$ \textbf{to} $i_{\max}$}
    \State Evaluate fitness \( f(\sigma) \) for each candidate \( \sigma \in P \)
    \State Select the top half of the candidates in \( P \) based on fitness, denoted as \( P_{\text{parents}} \)
    \State \( P_{\text{offspring}} \gets \emptyset \)
    \For{$j = 1$ \textbf{to} $\frac{p_{\max}}{2}$}
        \State Randomly select 2 parents \( \sigma_1, \sigma_2 \in P_{\text{parents}} \)
        \State \( \tau_1, \tau_2 \gets \text{Crossover}(\sigma_1, \sigma_2) \)
        \State \( P_{\text{offspring}} \gets P_{\text{offspring}} \cup \{\tau_1, \tau_2\} \)
    \EndFor
    \For {each \( \tau \in P_{\text{offspring}} \)}
        \State With probability \( m \), swap 2 random values in \( \tau \)
        \State With probability \( m \), replace a random value in \( \tau \) with a new value not already in \(\tau\)
    \EndFor
    \State \( P \gets P_{\text{offspring}} \)
    \State \( \sigma' \gets \text{argmin}_{\sigma \in P} f(\sigma) \)
    \If{\(f(\sigma') < f(\sigma^*)\)}
        \State \( \sigma^* \gets \sigma'\)
    \EndIf
\EndFor
\State \textbf{Return:} \(\sigma^*\)
\end{algorithmic}
\end{algorithm}

The mutation part of the algorithm consists of two independent mutation operators. The first mutation operator is randomly swapping two indices of the candidate solution. The second mutation operator is replacing a random value in the ordered set with another value which is not already in the solution. A solution cannot have duplicated values.

The crossover operator used in our approach is a random one-point crossover. However, a challenge arises because this crossover can create invalid offspring with duplicated elements if both parents have common elements. To address this issue, we implement a strategy to separate the values common to both parents from the other values.

Specifically, the values that are common to both parents will be preserved and not subjected to crossover. Next, we separate the remaining values (those not common to both parents) into two new sets of genes. These new sets will contain only unique values, ensuring that no duplicates are present. A normal one-point crossover can be perfomed on these new gene sets, creating two new sets of genes without any common values. After the crossover, reinsert the common values back into the respective offspring. This ensures that the offspring are valid and maintain the necessary elements from both parents without any duplication. Fig. \ref{fig:crossover} details an example using this crossover operator.

By using this method, we ensure that the resulting offspring are valid and retain genetic diversity from both parents while avoiding any duplicate values.

\begin{algorithm}[H]
\caption{Crossover Algorithm} \label{alg:crossover} 
\begin{algorithmic}[1]
\State \textbf{Input:} Two sequences of distinct values \(\sigma_1\), \(\sigma_2\).
\State \(D \gets \sigma_1 \cap \sigma_2\)
\State \(\sigma_1' \gets \sigma_1 \setminus D\)
\State \(\sigma_2' \gets \sigma_2 \setminus D\)
\State \(\tau_1' \text{, } \tau_2' \gets \text{OnePointCrossover}(\sigma_1' \text{, } \sigma_2')\)
\State \(\tau_1 \gets \tau_1' \cup D\)
\State \(\tau_2 \gets \tau_2' \cup D\)
\State \textbf{Return:} \(\tau_1 \text{, } \tau_2\)
\end{algorithmic}
\end{algorithm}

\begin{figure*}
    \centering
    \begin{tikzpicture}[scale=0.8, every node/.style={scale=0.8}]
        \node at (0, 4) {Parent 1: [1, 2, 3, 4, 5, 6]};
        \node at (0, 3) {Parent 2: [4, 5, 6, 7, 8, 9]};
        
        \node at (0, 2) {Common Values: \{4, 5, 6\}};
        \node at (0, 1.5) {Unique Values in Parent 1: \{1, 2, 3\}};
        \node at (0, 1) {Unique Values in Parent 2: \{7, 8, 9\}};
        
        \draw [->] (0, 0.5) -- (0, 0);
        \node at (0, -0.5) {Perform one-point crossover on unique values:};
        \node at (0, -1) {Crossover Point: 2};
        
        \node at (-4, -2) {Unique Offspring 1: [1, 2, 9]};
        \node at (4, -2) {Unique Offspring 2: [7, 8, 3]};
        
        \draw [->] (-4, -2.5) -- (-4, -3);
        \draw [->] (4, -2.5) -- (4, -3);
        \node at (-4, -3.5) {Offspring 1: [1, 2, 9, 4, 5, 6]};
        \node at (4, -3.5) {Offspring 2: [7, 8, 3, 4, 5, 6]};
    \end{tikzpicture}
    \caption{One-point crossover with common and unique values.}
    \label{fig:crossover}
\end{figure*}
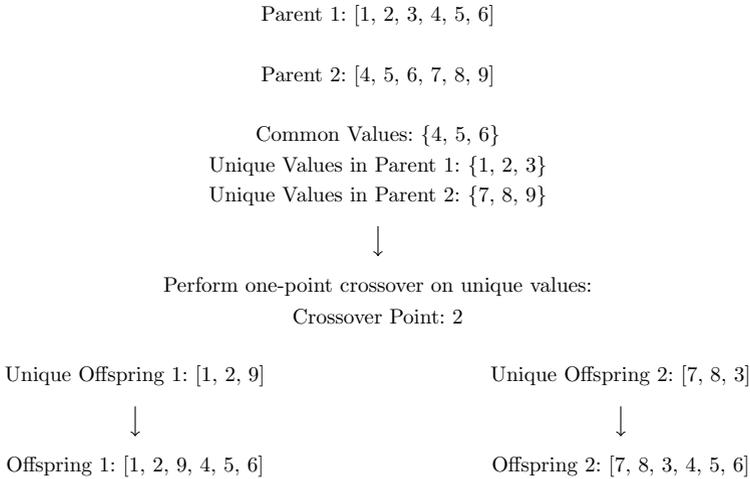
}

\subsection{Directed Evolution}

{
Directed evolution, a laboratory methodology wherein biological entities possessing desired characteristics are generated through iterative cycles of genetic diversification and library screening or selection, has emerged as a highly valuable and extensively utilized instrument in both fundamental and practical realms of biological research \cite{doi:10.1021/ar960017f, https://doi.org/10.1002/anie.201708408, Romero2009ExploringPF}.

\begin{algorithm}[H]
\caption{Directed Evolution for MMF} \label{alg:de} 
\begin{algorithmic}[1]
\State \textbf{Input:} Matrix \( \mA \) to factorize, number of resolution levels \( L \), number of indices chosen for each resolution level \( k \), size of the matrix \( n \), maximum population size \( p_{\max} \) (must be even), number of iterations \( i_{\max} \), the fitness function \( f \) representing the Frobenius error from the Stiefel manifold optimization algorithm.

\State Initialize the population \( P \) with \( p_{\max} \) random ordered sets of size \( L \) from the range \([1, n]\)
\State Initialize \( \sigma^* \) as a random solution
\For{$i = 1$ \textbf{to} $i_{\max}$}
    \State Evaluate fitness \( f(\sigma) \) for each candidate \( \sigma \in P \)
    \State Select the top half of the candidates in \( P \) based on fitness, denoted as \( P_{\text{parents}} \)
    \State \( P_{\text{offspring}} \gets \emptyset \)
    \For {each \( \sigma \in P_{\text{parents}} \)}
        \State \(\tau \gets \sigma\)
        \State Swap 2 random values in \( \tau \)
        \State Replace a random value in \( \tau \) with a new value not already in \(\tau\)
        \State \( P_{\text{offspring}} \gets P_{\text{offspring}} \cup \{\tau\} \)
    \EndFor
    \State \( P \gets P_{\text{parents}} \cup P_{\text{offspring}} \)
    \State \( \sigma' \gets \text{argmin}_{\sigma \in P} f(\sigma) \)
    \If{\(f(\sigma') < f(\sigma^*)\)}
        \State \( \sigma^* \gets \sigma'\)
    \EndIf
\EndFor
\State \textbf{Return:} \(\sigma^*\)
\end{algorithmic}
\end{algorithm}

This directed evolution algorithm uses the same mutation operators as the evolutionary algorithm. Mutation is performed on every member of the parent population. The parent population is part of the next generation. 

\begin{figure*}
  \begin{center}
  \includegraphics[width=\textwidth,page=1]{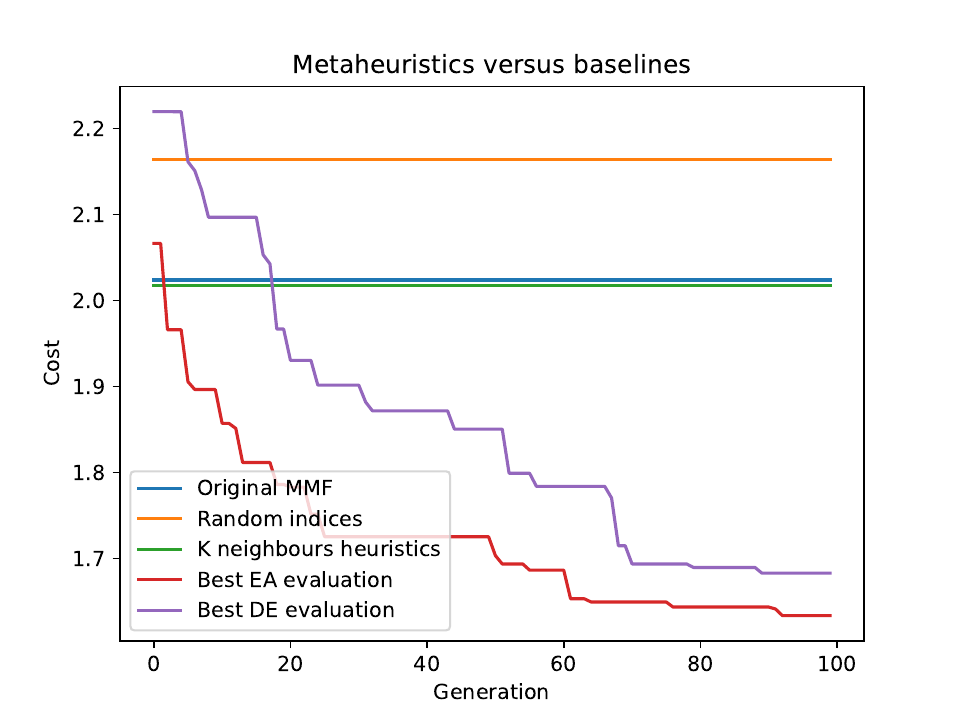}
  \caption{\label{fig:metaheuristics_convergence} Metaheuristics convergence for Karate Club data. Selection process based on Evolutionary Algorithm (EA) and Directed Evolution (DE) outperforms the original heuristics proposed by \cite{pmlr-v32-kondor14}.}
  \end{center}
\end{figure*}

\begin{table}[h!]
    \centering
    \begin{tabular}{|| c | c ||}
        \hline
        \textbf{Method} & \textbf{Runtime (seconds)} \\ 
        \hline
        Original MMF & 0.044 \\
        \hline
        Random indices MMF & 0.012 \\
        \hline
        Heuristics MMF & 0.022 \\
        \hline
        EA MMF & 115.726 \\
        \hline
        DE MMF & 4.641 \\
        \hline
    \end{tabular}
    \caption{Runtimes for different MMF methods.}
    \label{tab:MMFruntimes}
\end{table}

\begin{figure*}
\begin{center}
\begin{tabular}{ccc}
\includegraphics[scale=0.15]{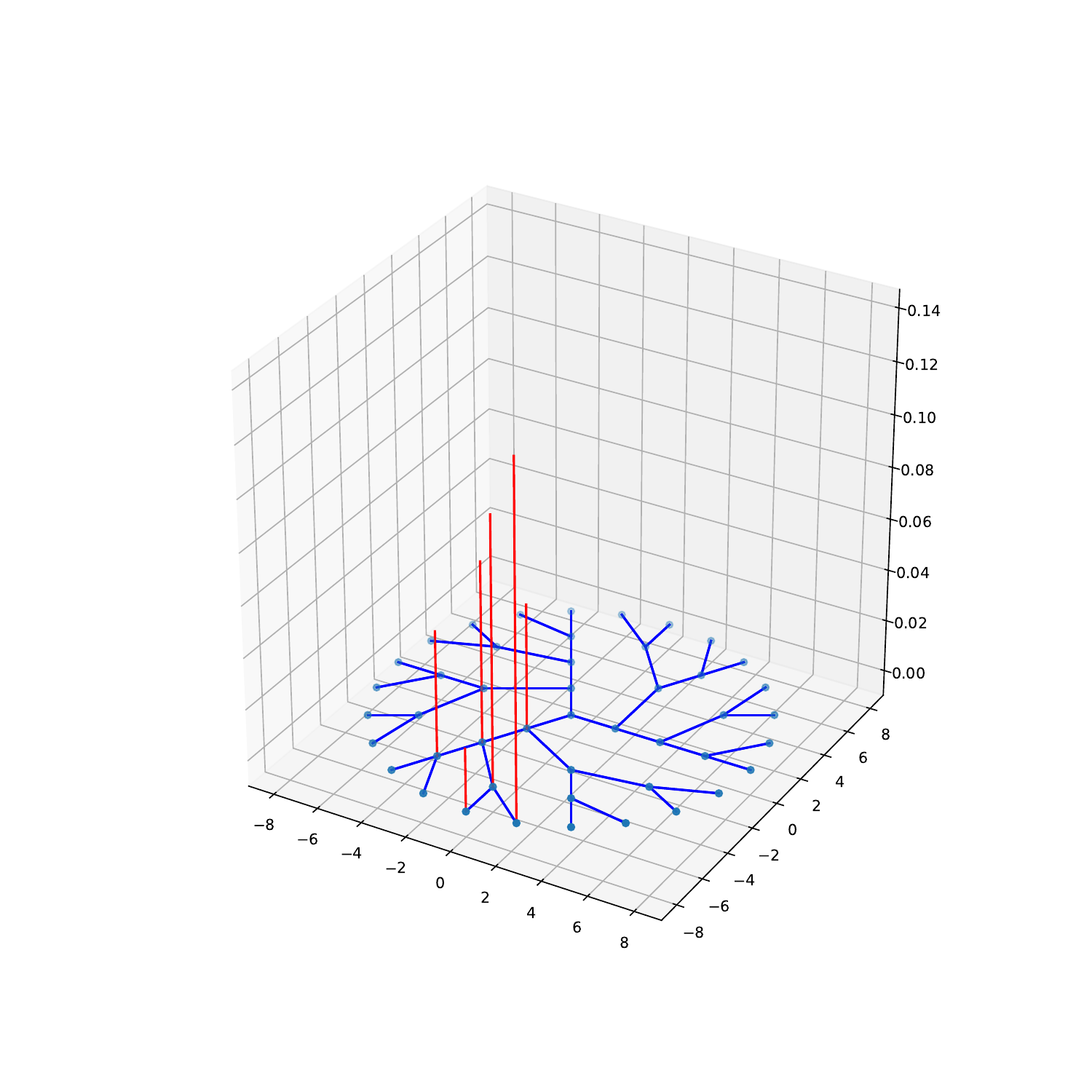}
\vspace{-15pt}
&
\includegraphics[scale=0.15]{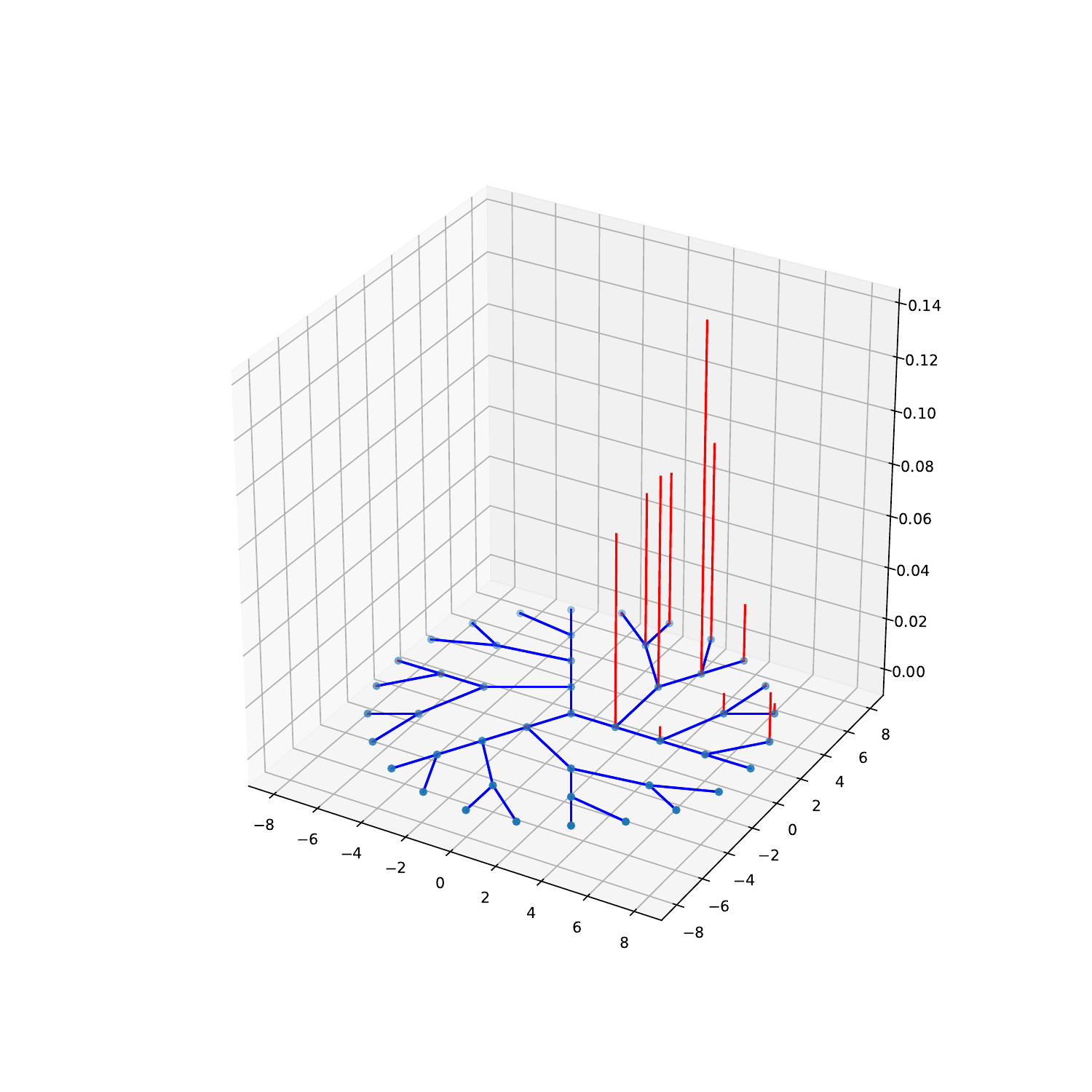}
\vspace{-2pt}
&
\includegraphics[scale=0.15]{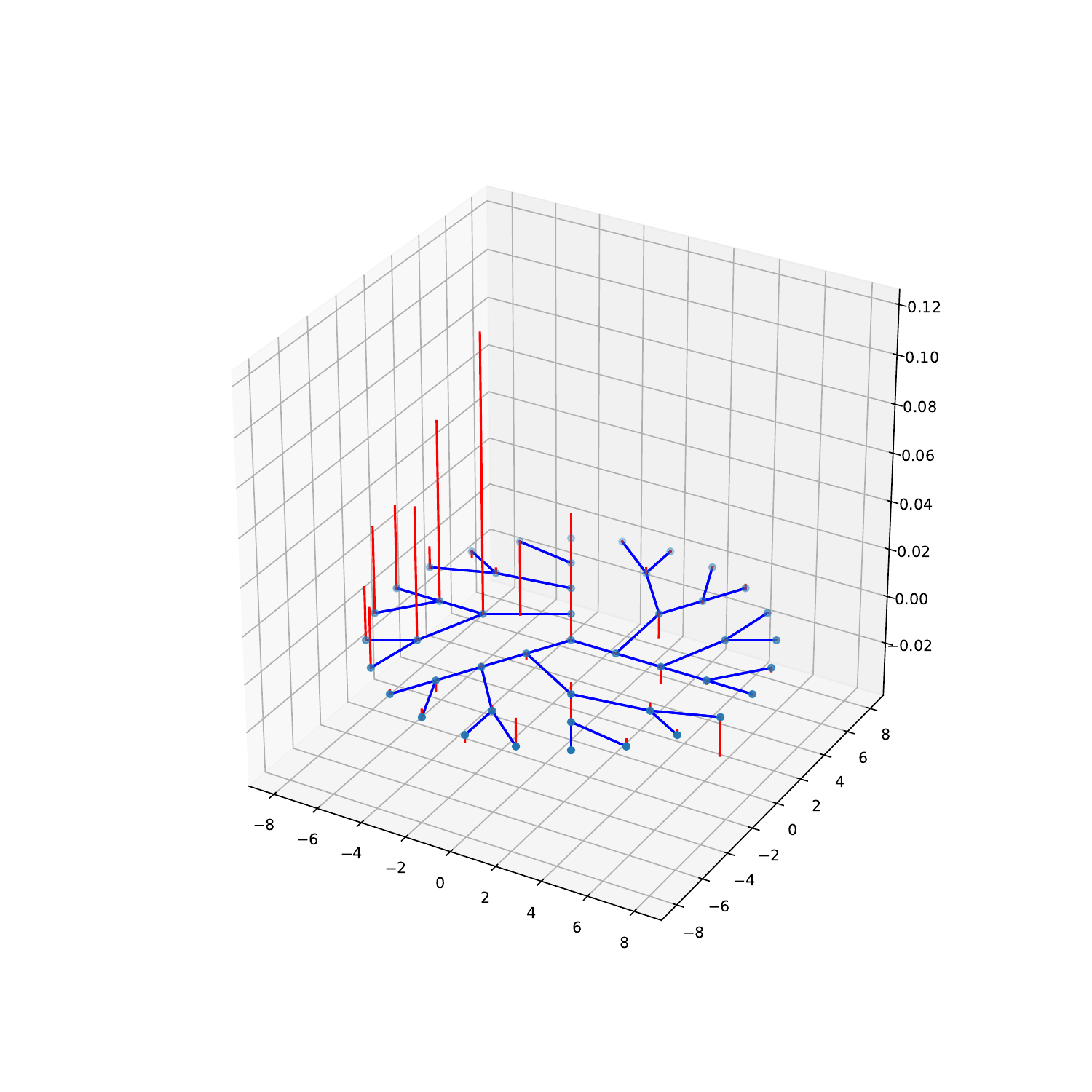}
\vspace{-2pt}
\\
$\ell = 1$
&
$\ell = 20$
&
$\ell = 39$
\end{tabular}
\end{center}
\caption{\label{fig:wavelets-visual} Visualization of some of the wavelets on the Cayley tree of 46 vertices. 
The low index wavelets (low $\ell$) are highly localized, whereas the high index ones are smoother and spread out over large
parts of the graph.}
\end{figure*}

Figure \ref{fig:metaheuristics_convergence} illustrates the convergence behavior of two metaheuristics applied to the Karate Club matrix, contrasted with various random and heuristic baselines. Notably, after 100 generations, the Evolutionary Algorithm (EA) surpasses all other baselines in performance. Additionally, Directed Evolution (DE) demonstrates effective MMF approximation. However, it is worth noting that EA requires significantly more time to reach convergence over 100 generations. This disparity in time consumption is attributed to the larger population size initialized in EA compared to DE. Other heuristic baselines, although quicker, fail to provide comparably accurate approximations as the two metaheuristics.

Figure \ref{fig:wavelets-visual} depicts the wavelet bases at different levels of resolution. The low index wavelets are localized since it teases out distinct local structures of the matrix, separating it into rough and smoother components. The higher wavelet bases are left only with the smooth part of the matrix.  This observation aligns with the interpretation provided for MMF in \ref{sec:mra}. 
 
}

\section{Wavelet Neural Networks on Graphs} \label{sec:networks}

\subsection{Motivation}

The eigendecomposition of the normalized graph Laplacian operator $\tilde{\mL} = \mU^T \mH \mU$ 
can be used as the basis of a graph Fourier transform. \cite{6494675} defines graph Fourier transform (GFT) on a graph $\mathcal{G} = (V, E)$ of a graph signal $\vf \in \mathbb{R}^n$ (that is understood as a function $f: V \rightarrow \mathbb{R}$ defined on the vertices of the graph) as $\hat{\vf} = \mU^T \vf$, and the inverse graph Fourier transform as $\vf = \mU \hat{\vf}$. Analogously to the classical Fourier transform, GFT provides a way to represent a graph signal in two domains: the vertex domain and the graph spectral domain; to filter graph signal according to smoothness; and to define the graph convolution operator, denoted as $*_{\mathcal{G}}$:
\begin{equation}
\vf *_{\mathcal{G}} \vg = \mU \big( (\mU^T \vg) \odot (\mU^T \vf) \big),
\label{eq:gft-conv}
\end{equation}
where $\vg$ denotes the convolution kernel, and $\odot$ is the element-wise Hadamard product. 
If we replace the vector $\mU^T \vg$ by a diagonal matrix $\tilde{\vg}$, then we can rewrite the Hadamard 
product in Eq.~(\ref{eq:gft-conv}) to matrix multiplication as $\mU \tilde{\vg} \mU^T \vf$ (that is understood as filtering the signal $\vf$ by the filter $\tilde{\vg}$). Based on GFT, \cite{ae482107de73461787258f805cf8f4ed} and \cite{NIPS2016_04df4d43} construct convolutional neural networks (CNNs) learning on spectral domain for discrete structures such as graphs. However, there are two fundamental limitations of GFT:
\begin{itemize}
\item High computational cost: eigendecomposition of the graph Laplacian has complexity $O(n^3)$, 
and ``Fourier transform'' itself involves multiplying the signal with a dense matrix of eigenvectors. 
\item The graph convolution is not localized in the vertex domain, even 
if the graph itself has well defined local communities.
\end{itemize}
To address these limitations, we propose a modified spectral graph network based on the MMF wavelet 
basis rather than the eigenbasis of the Laplacian. 
This has the following advantages: 
(i) the wavelets are generally localized in both vertex domain and frequency, 
(ii) the individual basis transforms are sparse, and 
(iii) MMF provides a computationally efficient way of decomposing graph signals into components at 
different granularity levels and an excellent basis for sparse approximations.

\newpage 
\subsection{Network construction}


{
In this section, we define a convolution layer based on the wavelet bases from the MMF. This construction is inspired mainly from the GFT defined in \cite{ae482107de73461787258f805cf8f4ed} and the connection between MMF and multiresolution analysis suggested in \cite{pmlr-v32-kondor14}. 

\cite{pmlr-v32-kondor14} demonstrated that MMF aligns with the classical theory of multiresolution analysis (MRA), transitioning from the real line \cite{192463} to discrete spaces. In MRA of a symmetric matrix \(A \in \mathbb{R}^{n \times n}\), the goal is to identify a sequence of subspaces:
\begin{equation}
\sV_{L} \subset \dots \subset \sV_2 \subset \sV_1 \subset \sV_0
\label{eq:subspace-sequence-6}
\end{equation}
The process is akin to an iterative refinement, where each subspace \(\sV_\ell\) is decomposed into an orthogonal sum: \(\sV_\ell = \sV_{\ell + 1} \oplus \sW_{\ell + 1}\), comprising a smoother part \(\sV_{\ell + 1}\) (the approximation space) and a rougher part \(\sW_{\ell + 1}\) (the detail space) (refer to Fig.~\ref{fig:subspaces}). Within each subspace \(\sV_\ell\), there exists an orthonormal basis denoted by \(\Phi_\ell \triangleq \{\phi_m^\ell\}_m\), where each basis function is referred to as a \textit{father} wavelet. Similarly, the complementary space \(\sW_\ell\) possesses an orthonormal basis denoted by \(\Psi_\ell \triangleq \{\psi_m^\ell\}_m\), with each basis function termed a \textit{mother} wavelet (see \ref{sec:mra} for a deeper exploration of MMF's interpretation within multiresolution analysis). Based on these wavelet bases, we can define a wavelet transform for a symmetric matrix.
}

In the case $\mA$ is the normalized graph Laplacian of a graph $\mathcal{G} = (V, E)$, the wavelet transform (up to level $L$) expresses a graph signal (function over the vertex domain) $f: V \rightarrow \mathbb{R}$, without loss of generality $f \in \sV_0$, as:
$$f(v) = \sum_{\ell = 1}^L \sum_m \alpha_m^\ell \psi_m^\ell(v) + \sum_m \beta_m \phi_m^L(v), \ \ \ \ \text{for each} \ \ v \in V,$$ 
where $\alpha_m^\ell = \langle f, \psi_m^\ell \rangle$ and $\beta_m = \langle f, \phi_m^L \rangle$ are the wavelet coefficients. At each level, a set of coordinates $\sT_\ell \subset \sS_{\ell -  1}$ are selected to be the wavelet indices, and then to be eliminated from the active set by setting $\sS_\ell = \sS_{\ell - 1} \setminus \sT_\ell$ (see Section \ref{sec:ea-problem}). Practically, we make the assumption that we only select $1$ wavelet index for each level that results in a single mother wavelet $\psi^\ell = [\mA_\ell]_{i^*, :}$ where $i^*$ is the selected index (see Section \ref{sec:ea-problem}). We get exactly $L$ mother wavelets $\overline{\psi} = \{\psi^1, \psi^2, \dots, \psi^L\}$. On the another hand, the active rows of $\mH = \mA_L$ make exactly $N - L$ father wavelets $\overline{\phi} = \{\phi^L_m = \mH_{m, :}\}_{m \in \sS_L}$. In total, a graph of $N$ vertices has exactly $N$ wavelets (both mothers and fathers). 

Analogous to the convolution based on GFT \cite{ae482107de73461787258f805cf8f4ed}, each convolution layer $k = 1, .., K$ of our wavelet network transforms an input vector $\vf^{(k - 1)}$ of size $\lvert V \rvert \times F_{k - 1}$ into an output $\vf^{(k)}$ of size $\lvert V \rvert \times F_k$ as
\begin{equation}
\vf^{(k)}_{:, j} = \sigma \bigg( \mW \sum_{i = 1}^{F_{k - 1}} \vg^{(k)}_{i, j} \mW^T \vf^{(k - 1)}_{:, i} \bigg) \ \ \ \ \text{for} \ \ j = 1, \dots, F_k,
\label{eq:wavevlet-conv}
\end{equation}
where $\mW$ is our wavelet basis matrix as we concatenate $\overline{\phi}$ and $\overline{\psi}$ column-by-column, $\vg^{(k)}_{i, j}$ is a parameter/filter in the form of a diagonal matrix learned in spectral domain similar to the filter used in the original GFT construction \cite{ae482107de73461787258f805cf8f4ed}, and $\sigma$ is an element-wise linearity (e.g., ReLU, sigmoid, etc.). Each layer transforms the input features \(f^{(k - 1)}\) into a different domain, performs filtering operations defined by the parameter \(g^{(k)}\) before reverting features back to the original domain. The training process is responsible for tuning the filter \(g^{(k)}\) to extract relevant information from the input graph signal.

\section{Experiments} \label{sec:Experiments}

\subsection{Molecular graphs classification}

\begin{table*}
\begin{center}
\begin{tabular}{||l | c | c | c | c | c ||}
\hline
\textbf{Method} & \textbf{MUTAG} & \textbf{PTC} & \textbf{PROTEINS} & \textbf{NCI1} \\
\hline
\hline
DGCNN {\small \cite{Zhang2018AnED}} & 85.83 $\pm$ 1.7 & 58.59 $\pm$ 2.5 & 75.54 $\pm$ 0.9 & 74.44 $\pm$ 0.5 \\
\hline
PSCN {\small \cite{Niepert2016}} & 88.95 $\pm$ 4.4 & 62.29 $\pm$ 5.7 & 75 $\pm$ 2.5 & 76.34 $\pm$ 1.7 \\
\hline
DCNN {\small \cite{10.5555/3157096.3157320}} & N/A & N/A & 61.29 $\pm$ 1.6 & 56.61 $\pm$ 1.0 \\
\hline
CCN {\small \cite{HyEtAl2018}} & \textbf{91.64 $\pm$ 7.2} & \textbf{70.62 $\pm$ 7.0} & N/A & 76.27 $\pm$ 4.1 \\
\hline
GK {\small \cite{pmlr-v5-shervashidze09a}} & 81.39 $\pm$ 1.7 & 55.65 $\pm$ 0.5 & 71.39 $\pm$ 0.3 & 62.49 $\pm$ 0.3 \\
\hline
RW {\small \cite{10.5555/1756006.1859891}} & 79.17 $\pm$ 2.1 & 55.91 $\pm$ 0.3 & 59.57 $\pm$ 0.1 & N/A \\
\hline
PK {\small \cite{Neumann2016}} & 76 $\pm$ 2.7 & 59.5 $\pm$ 2.4 & 73.68 $\pm$ 0.7 & 82.54 $\pm$ 0.5 \\
\hline
WL {\small \cite{JMLR:v12:shervashidze11a}} & 84.11 $\pm$ 1.9 & 57.97 $\pm$ 2.5 & 74.68 $\pm$ 0.5 & \textbf{84.46 $\pm$ 0.5} \\
\hline
IEGN {\small \cite{maron2018invariant}} & 84.61 $\pm$ 10 & 59.47 $\pm$ 7.3 & 75.19 $\pm$ 4.3 & 73.71 $\pm$ 2.6 \\
\hline
\hline
\textbf{MMF} & 86.31 $\pm$ 9.47 & 67.99 $\pm$ 8.55 & \textbf{78.72 $\pm$ 2.53} & 71.04 $\pm$ 1.53 \\
\hline
\end{tabular}
\end{center}
\caption{\label{tbl:graph-classification} Molecular graphs classification. Baseline results are taken from \cite{maron2018invariant}.}
\end{table*}

We trained and evaluated our wavelet networks (WNNs) on standard graph classification benchmarks including four bioinformatics datasets: 
(1) MUTAG, which is a dataset of 188 mutagenic aromatic and heteroaromatic nitro compounds with 7 discrete 
labels \cite{doi:10.1021/jm00106a046}; (2) PTC, which consists of 344 chemical compounds with 19 discrete 
labels that have been tested for positive or negative toxicity in lab rats \cite{10.1093/bioinformatics/btg130}; 
(3) PROTEINS, which contains 1,113 molecular graphs with binary labels, 
where nodes are secondary structure elements (SSEs) and there is an edge between two nodes if they are 
neighbors in the amino-acid sequence or in 3D space \cite{ProteinKernel}; 
(4) NCI1, which has 4,110 compounds with binary labels, each screened for activity against small cell 
lung cancer and ovarian cancer lines \cite{NCIDataset}. 
Each molecule is represented by an adjacency matrix, and we represent each atomic type as a one-hot vector 
and use them as the node features.

We factorize all normalized graph Laplacian matrices in these datasets by MMF with $K = 2$ to obtain the wavelet bases. 
Again, MMF wavelets are \textbf{sparse} and suitable for fast transform via sparse matrix multiplication. 
{
The sparsity of wavelet bases, as shown in Table \ref{tab:sparsity_bases}, highlights a significant compression compared to the Fourier bases derived from the eigendecomposition of the graph Laplacian.

\begin{table}[h!]
    \centering
    \begin{tabular}{|| c | c | c ||}
        \hline
        \textbf{Dataset} & \textbf{Fourier bases}  & \textbf{Wavelet bases}\\ 
        \hline
        MUTAG & 99.71\% & 19.23\%\\
        \hline
        PTC & 99.30\% & 18.18\%\\
        \hline
        PROTEINS & 99.33\% & 2.26\%\\
        \hline
        NCI1 & 99.04\% & 11.43\%\\
        \hline
    \end{tabular}
    \caption{ Sparsity bases (i.e. percentage of non-zeros).}
    \label{tab:sparsity_bases}
\end{table}

}

Our WNNs contain 6 layers of spectral convolution, 32 hidden units for each node, 
and are trained with 256 epochs by Adam optimization with an initial learning rate of $10^{-3}$. 
We follow the evaluation protocol of 10-fold cross-validation from \cite{Zhang2018AnED}. 
We compare our results to several deep learning methods and popular graph kernel methods. Baseline results are taken from \cite{maron2018invariant}.

For graph kernel methods, we compare our model with four popular approaches: the graphlet kernel (GK), the random walk kernel (RW), the propagation kernel (PK), and the Weisfeiler-Lehman subtree kernel (WL). Each of these methods employs unique strategies for capturing graph structure and similarity. The graphlet kernel focuses on counting occurrences of small subgraphs, known as graphlets, to measure graph similarity. In contrast, the random walk kernel simulates random walks on graphs and compares the distributions of these walks to compute similarity. The propagation kernel, on the other hand, considers the propagation of labels or information through the graph to determine similarity. Lastly, the Weisfeiler-Lehman subtree kernel compares graphs based on the structural information captured by subtrees rooted at each node, iteratively refining node representations. Our results demonstrate that our method outperforms all other kernel methods on three datasets: MUTAG, PTC, and PROTEINS. However, the WL kernel achieves the best performance on the NCI1 dataset.

For deep learning methods, we compare our model with five established approaches from the literature: DGCNN, PSCN, DCNN, CCN, and IEGN. Thse deep learning methods are more closely related to our WWN, as they all leverage neural network architectures for graph analysis. Our method ranks 3rd on the MUTAG dataset, 2nd on the PTC dataset, 1st on the PROTEINS dataset, and performs the worst on the NCI1 dataset.

Our WNNs outperform 6/8, 7/8, 8/8, and 2/8 baseline methods on MUTAG, PTC, PROTEINS, and NCI1, 
respectively (see Table \ref{tbl:graph-classification}).

Despite these promising results, our WNN model has some limitations that we need to address in future work. One significant limitation is its performance on the NCI1 dataset, where it doese not perform as well as other methods. This suggests that our model might struggle with certain types of graphs or larger datasets. Especially considering that among the four datasets, NCI1 has the largest number of unique atom types (37, compared to less than 20 in the other three datasets) and it is also the largest dataset in terms of size. Nonetheless, further experimentation is needed to confirm the types of graph for which the model's accuracy is suboptimal.


\subsection{Node classification on citation graphs}

To further evaluate the wavelet bases returned by our learnable MMF algorithm, we construct our wavelet networks 
(WNNs) as in Sec.~\ref{sec:networks} and apply it to the task of node classification on two citation graphs, 
Cora $(N = 2,708)$ and Citeseer $(N = 3,312)$ \cite{PSen2008} 
in which nodes and edges represent documents and citation links. 

In node classification tasks, assume the number of classes is $C$, 
the set of labeled nodes is $V_{\text{label}}$, 
and we are given a normalized graph Laplacian $\tilde{L}$ and an input node feature matrix $\vf^{(0)}$. 
First of all, we apply our MMF learning algorithm to factorize $\tilde{L}$ and produce our wavelet basis matrix $\mW$. Then, we construct our wavelet network as a multi-layer CNNs with each convolution is defined as in Eq.~(\ref{eq:wavevlet-conv}) that transforms $\vf^{(0)}$ into $\vf^{(K)}$ after $K$ layers. The top convolution layer $K$-th returns exactly $F_K = C$ features and uses softmax instead of the nonlinearity $\sigma$ for each node. The loss is the cross-entropy error over all labeled nodes as:
\begin{equation}
\mathcal{L} = - \sum_{v \in V_{\text{label}}} \sum_{c = 1}^C \vy_{v, c} \ln \vf^{(K)}_{v, c},
\label{eq:entropy-loss}
\end{equation}
where $\vy_{v, c}$ is a binary indicator that is equal to $1$ if node $v$ is labeled with class $c$, and $0$ otherwise. The set of weights $\{\vg^{(k)}\}_{k = 1}^K$ are trained using gradient descent optimizing the loss in Eq.~(\ref{eq:entropy-loss}).

Each document in Cora and Citeseer has an associated feature vector (of length  $1,433$ resp.~$3,703$)   
computed from word frequencies, and is classified into one of $7$ and $6$ classes, respectively. 
We factorize the normalized graph Laplacian by learnable MMF with $K = 16$ to obtain the wavelet bases. 
The resulting MMF wavelets are sparse, which makes it possible to run a fast transform 
on the node features by sparse matrix multiplication: only $4.69\%$ and $15.25\%$ of elements are non-zero in Citeseer and Cora, 
respectively. 
In constrast, Fourier bases given by eigendecomposition of the graph Laplacian are completely dense 
($100\%$ of elements are non-zero). 
We evaluate our WNNs with 3 different random splits of train/validation/test: (1) $20\%$/$20\%$/$60\%$ 
denoted as MMF$_1$, (2) $40\%$/$20\%$/$40\%$ denoted as MMF$_2$, and (3) $60\%$/$20\%$/$20\%$ denoted as MMF$_3$. 
The WNN learns to encode the whole graph with $6$ layers of spectral convolution and $100$ hidden dimensions 
for each node. During training, the network is only trained to predict the node labels in the training set. 
Hyperparameter searching is done on the validation set. 
The number of epochs is $256$ and we use the Adam optimization method \cite{Kingma2015} 
with learning rate $\eta = 10^{-3}$. 
We report the final test accuracy for each split in Table \ref{tbl:node-classification}.

We compare with several traditional methods and deep learning methods including other spectral graph convolution networks such as Spectral CNN, and graph wavelet neural networks (GWNN). Baseline results are taken from \cite{xu2018graph}. Our wavelet networks perform competitively against state-of-the-art methods in the field.

\begin{table}[h]
\begin{center}
\begin{tabular}{||l | c | c ||}
\hline
\textbf{Method} & \textbf{Cora} & \textbf{Citeseer} \\
\hline
MLP & 55.1\% & 46.5\% \\
ManiReg {\cite{JMLR:v7:belkin06a}} & 59.5\% & 60.1\% \\
SemiEmb {\cite{10.1145/1390156.1390303}} & 59.0\% & 59.6\% \\
LP {\cite{10.5555/3041838.3041953}} & 68.0\% & 45.3\% \\
DeepWalk {\cite{10.1145/2623330.2623732}}& 67.2\% & 43.2\% \\
ICA {\cite{Getoor2005}} & 75.1\% & 69.1\% \\
Planetoid {\cite{10.5555/3045390.3045396}} & 75.7\% & 64.7\% \\
\hline
Spectral CNN {\cite{ae482107de73461787258f805cf8f4ed}} & 73.3\% & 58.9\% \\
ChebyNet {\cite{NIPS2016_04df4d43}} & 81.2\% & 69.8\% \\
GCN {\cite{Kipf:2016tc}} & 81.5\% & 70.3\% \\
MoNet {\cite{DBLP:journals/corr/MontiBMRSB16}} & 81.7\% & N/A \\
GWNN {\cite{xu2018graph}} & 82.8\% & 71.7\% \\
\hline
\textbf{MMF}$_1$ & \textbf{84.35\%} & 68.07\% \\
\textbf{MMF}$_2$ & \textbf{84.55\%} & \textbf{72.76\%} \\
\textbf{MMF}$_3$ & \textbf{87.59\%} & \textbf{72.90\%} \\
\hline
\end{tabular}
\end{center}
\caption{\label{tbl:node-classification} Node classification on citation graphs. Baseline results are taken from \cite{xu2018graph}.}
\end{table}

\subsection{Matrix factorization} \label{sec:approx-exp}

We evaluate the performance of our MMF learning algorithm in comparison with the original greedy algorithm \cite{pmlr-v32-kondor14} and the Nystr\"{o}m method \cite{pmlr-v28-gittens13} in the task of matrix factorization on 3 datasets: (i) normalized graph Laplacian of the Karate club network ($N = 34$, $E = 78$) \cite{Karate}; (ii) a Kronecker product matrix ($N = 512$), $\mathcal{K}_1^n$, of order $n = 9$, where $\mathcal{K}_1 = ((0, 1), (1, 1))$ is a $2 \times 2$ seed matrix \cite{JMLR:v11:leskovec10a}; and (iii) normalized graph Laplacian of a Cayley tree or Bethe lattice with coordination number $z = 4$ and $4$ levels of depth ($N = 161$). The rotation matrix size $K$ are $8$, $16$ and $8$ for Karate, Kronecker and Cayley, 
respectively. Meanwhile, the original greedy MMF is limited to $K = 2$ and implements an exhaustive search to find an optimal pair of indices for each rotation. For both versions of MMF, we drop $c = 1$ columns after each rotation, 
which results in a final core size of $d_L = N - c \times L$. 
The exception is for the Kronecker matrix ($N = 512$), our learning algorithm drops up to $8$ columns (for example, $L = 62$ and $c = 8$ results into $d_L = 16$) to make sure that the number of learnable parameters $L \times K^2$ is much smaller the matrix size $N^2$. 
Our learning algorithm compresses the Kronecker matrix down to $6-7\%$ of its original size. The details of efficient training reinforcement learning with the policy networks implemented by GNNs are included in the Appendix.

For the baseline of Nystr\"{o}m method, we randomly select, by uniform sampling without replacement, the same number $d_L$ columns $\mC$ from $\mA$ and take out $\mW$ as the corresponding $d_L \times d_L$ submatrix of $\mA$. The Nystr\"{o}m method approximates $\mA \approx \mC\mW^{\dagger}\mC^T$. We measure the approximation error in Frobenius norm. Figure \ref{fig:matrix} shows our MMF learning algorithm consistently outperforms the original greedy algorithm and the Nystr\"{o}m baseline given the same number of active columns, $d_L$. 

\begin{figure*}
\begin{center}
\includegraphics[scale=0.36]{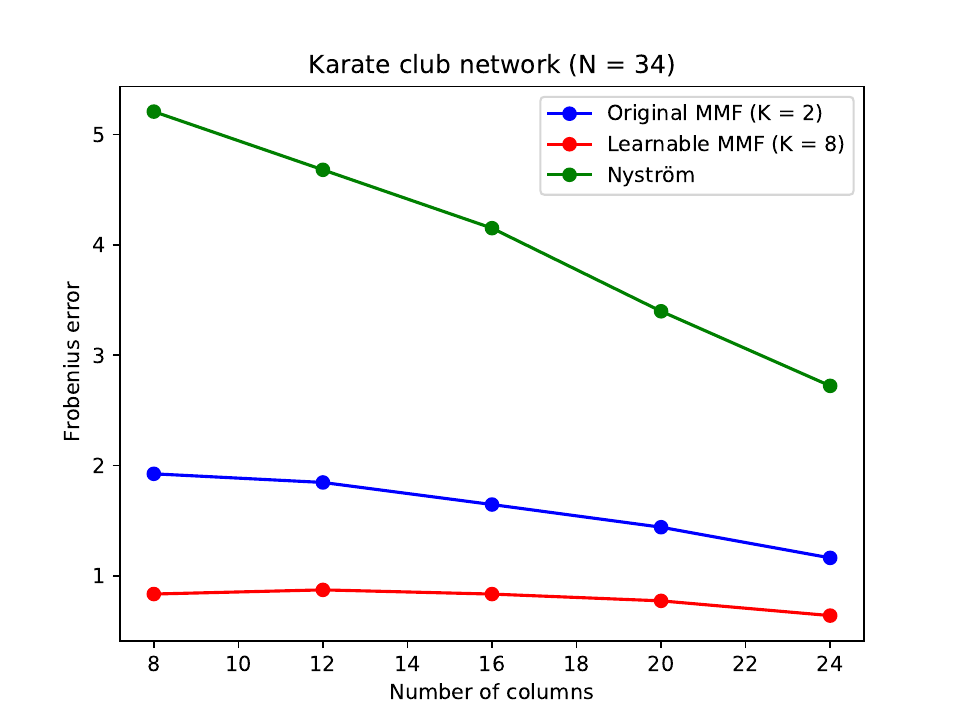} 
\includegraphics[scale=0.36]{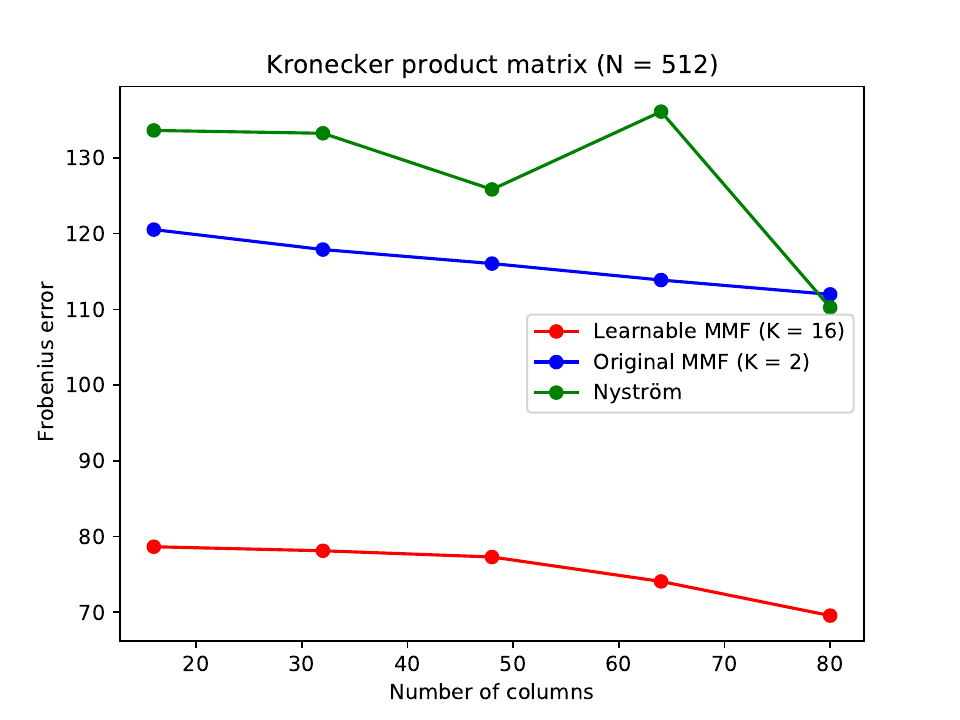} 
\includegraphics[scale=0.36]{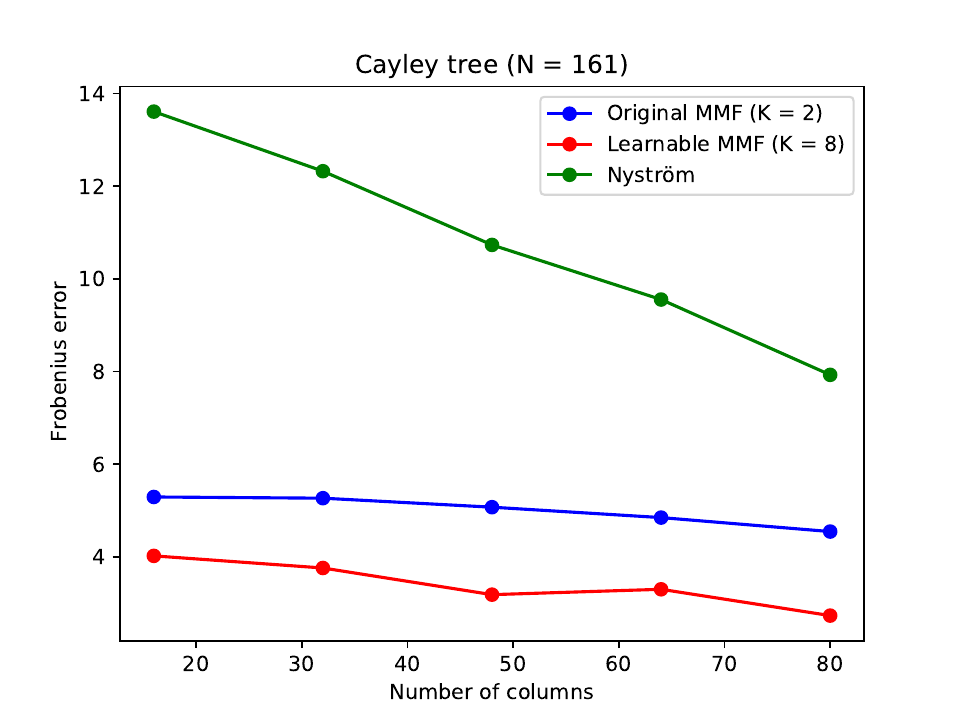}
\end{center}
\caption{\label{fig:matrix} Matrix factorization for the Karate network (left), Kronecker matrix (middle), and Cayley tree (right). 
Our learnable MMF consistently outperforms the classic greed methods.}
\end{figure*}

\section{Software}

We implemented our learning algorithm for MMF and the wavelet networks by PyTorch deep learning framework \citep{NEURIPS2019_bdbca288}. We released our implementation at 
\url{https://github.com/HySonLab/LearnMMF/}.

\section{Conclusions} \label{sec:Conclusion}

In this paper we introduced a general algorithm based on Stiefel manifold optimization and  evolutionary metaheuristics (e.g., Evolutionary Algorithm and Directed Evolution) to optimize Multiresolution Matrix Factorization (MMF). 
We find that the resulting learnable MMF consistently outperforms the existing greedy and heuristic 
MMF algorithms in factorizing and approximating hierarchical matrices. 
Based on the wavelet basis returned from our learning algorithm, we define a corresponding 
notion of spectral convolution 
and construct a wavelet neural network for graph learning problems. 
Thanks to the sparsity of the MMF wavelets, the wavelet network can be 
efficiently implemented with sparse matrix multiplication. 
We find that this combination of learnable MMF factorization and spectral wavelet network yields competitive results on standard node classification and molecular graph classification.


\bibliography{paper}


\clearpage
\begin{appendices}

\section{Notation} \label{sec:Notation}

We define $[n] = \{1, 2, \dots, n\}$ as the set of the first $n$ natural numbers. We denote $\mI_n$ as the $n$ dimensional identity matrix. The group of $n$ dimensional orthogonal matrices is $\sS\sO(n)$. $\sA \cupdot \sB$ will denote the disjoint union of two sets $\sA$ and $\sB$, therefore $\sA_1 \cupdot \sA_2 \cupdot \dots \cupdot \sA_k = \sS$ is a partition of $\sS$. 

Given a matrix $\mA \in \mathbb{R}^{n \times n}$ and two sequences of indices $\vi = (i_1, \dots, i_k) \in [n]^k$ and $\vj = (j_1, \dots, j_k) \in [n]^k$ assuming that $i_1 < i_2 < \dots < i_k$ and $j_1 < j_2 < \dots < j_k$, $\mA_{\vi, \vj}$ will be the $k \times k$ matrix with entries $[\mA_{\vi, \vj}]_{x, y} = \mA_{i_x, j_y}$. Furthermore, $\mA_{i, :}$ and $\mA_{:, j}$ denote the $i$-th row and the $j$-th column of $\mA$, respectively. Given $\mA_1 \in \mathbb{R}^{n_1 \times m_1}$ and $\mA_2 \in \mathbb{R}^{n_2 \times m_2}$, $\mA_1 \oplus \mA_2$ is the $(n_1 + n_2) \times (m_1 + m_2)$ dimensional matrix with entries
$$[\mA_1 \oplus \mA_2]_{i, j} = 
\begin{cases} 
[\mA_1]_{i, j} & \text{if} \ \ i \leq n_1 \ \ \text{and} \ \ j \leq m_1 \\
[\mA_2]_{i - n_1, j - m_1} & \text{if} \ \ i > n_1 \ \ \text{and} \ \ j > m_1 \\
0 & \text{otherwise.}
\end{cases}$$
A matrix $\mA$ is said to be block diagonal if it is of the form
\begin{equation}
\mA = \mA_1 \oplus \mA_2 \oplus \dots \oplus \mA_p
\label{eq:block-matrix}
\end{equation}
for some sequence of smaller matrices $\mA_1, \dots, \mA_p$. For the generalized block diagonal matrix, we remove the restriction that each block in (\ref{eq:block-matrix}) must involve a contiguous set of indices, and introduce the notation
\begin{equation}
\mA = \oplus_{(i_1^1, \dots, i_{k_1}^1)} \mA_1 \oplus_{(i_1^2, \dots, i_{k_2}^2)} \mA_2 \dots \oplus_{(i_1^p, \dots, i_{k_p}^p)} \mA_p
\label{eq:block-matrix-index}
\end{equation}

\noindent in which

$$
\mA_{a, b} =
\begin{cases}
[\mA_u]_{q, r} & \text{if} \ \ i_q^u = a \ \ \text{and} \ \ i^u_r = b \ \ \text{for some} \ \ u, q, r, \\
0 & \text{otherwise.}
\end{cases}
$$

We will sometimes abbreviate expressions like (\ref{eq:block-matrix-index}) by dropping the first \(\oplus\) operator and its indices.

Here is an example illustrating the notation used in \ref{eq:block-matrix-index}. Consider the following matrices:
\[
\mA_1 = \begin{pmatrix}
1 & 2 \\
3 & 4
\end{pmatrix}, \quad
\mA_2 = \begin{pmatrix}
5 & 6 \\
7 & 8
\end{pmatrix}
\]

We construct a generalized block diagonal matrix $\mA$ using the indices:

\begin{itemize}
    \item For $\mA_1$: rows and columns $(1, 3)$
    \item For $\mA_2$: rows and columns $(2, 4)$
\end{itemize}

Using the notation from \ref{eq:block-matrix-index}:
\[
\mA = \oplus_{(1, 3)} \mA_1 \oplus_{(2, 4)} \mA_2
\]

The resulting $4 \times 4$ matrix $\mA$ is:
\[
\mA = \begin{pmatrix}
1 & 0 & 2 & 0 \\
0 & 5 & 0 & 6 \\
3 & 0 & 4 & 0 \\
0 & 7 & 0 & 8
\end{pmatrix}
\]

Here, $\mA_1$ is placed in the positions corresponding to rows and columns $(1, 3)$, and $\mA_2$ is placed in the positions corresponding to rows and columns $(2, 4)$, with all other entries being zero.

The \textbf{Kronecker tensor product} $\mA_1 \otimes \mA_2$ of two matrices $\mA_1 \in \mathbb{R}^{n_1 \times m_1}$ and $\mA_2 \in \mathbb{R}^{n_2 \times m_2}$ is an $n_1n_2 \times m_1m_2$ matrix constructed as follows:
\[
[\mA_1 \otimes \mA_2]_{(i_1 - 1)n_2 + i_2, (j_1 - 1)m_2 + j_2} = [\mA_1]_{i_1, j_1} \cdot [\mA_2]_{i_2, j_2}.
\]

This means that each element of $\mA_1$ is multiplied by the entire matrix $\mA_2$, and the resulting blocks are arranged in the same relative positions as the elements of $\mA_1$.

To generalize, for $p$ matrices $\mA_1, \mA_2, \ldots, \mA_p$, the Kronecker product is denoted as $\mA_1 \otimes \mA_2 \otimes \dots \otimes \mA_p$. When we take the Kronecker product of a single matrix $\mA$ with itself $p$ times, we write this as $\mA^{\otimes p} = \mA \otimes \mA \otimes \dots \otimes \mA$.

A matrix $\mA \in \mathbb{R}^{n \times n}$ is called \textbf{skew-symmetric} (or \textbf{anti-symmetric}) if it satisfies the condition $\mA^T = -\mA$. This means that the transpose of $\mA$ is equal to its negative, i.e., $\mA_{ij} = -\mA_{ji}$ for all $i, j$. Skew-symmetric matrices have zeros on their diagonal since $\mA_{ii} = -\mA_{ii}$ implies $\mA_{ii} = 0$.

The \textbf{Euclidean inner product} between two matrices $\mA \in \mathbb{R}^{m \times n}$ and $\mB \in \mathbb{R}^{m \times n}$ is defined as:
\[
\langle \mA, \mB \rangle = \sum_{j, k} \mA_{j, k} \mB_{j, k} = \text{trace}(\mA^T \mB).
\]
This inner product is a natural extension of the dot product for vectors, summing the products of corresponding elements of the matrices.

The \textbf{Frobenius norm} of a matrix $\mA$ is given by:
\[
\|\mA\|_F = \sqrt{\sum_{i, j} \mA_{i, j}^2}.
\]

This norm measures the ``size'' of a matrix by considering the square root of the sum of the squares of all its entries. It is analogous to the Euclidean norm for vectors, providing a single number that reflects the overall magnitude of the matrix's elements.

\section{Multiresolution Matrix Factorization} \label{sec:MMF}

\subsection{Background} \label{sec:Background}

Most commonly used matrix factorization algorithms, such as principal component analysis (PCA), singular value decomposition (SVD), or non-negative matrix factorization (NMF) are inherently single-level algorithms. Saying that a symmetric matrix $\mA \in \mathbb{R}^{n \times n}$ is of rank $r \ll n$ means that it can be expressed in terms of a dictionary of $r$ mutually orthogonal unit vectors $\{u_1, u_2, \dots, u_r\}$ in the form
$$\mA = \sum_{i = 1}^r \lambda_i u_i u_i^T,$$
where $u_1, \dots, u_r$ are the normalized eigenvectors of $A$ and $\lambda_1, \dots, \lambda_r$ are the corresponding eigenvalues. This is the decomposition that PCA finds, and it corresponds to factorizing $\mA$ in the form
\begin{equation}
\mA = \mU^T \mH \mU,
\label{eq:eigen}
\end{equation}
where $\mU$ is an orthogonal matrix and $\mH$ is a diagonal matrix with the eigenvalues of $\mA$ on its diagonal. The drawback of PCA is that eigenvectors are almost always dense, while matrices occuring in learning problems, especially those related to graphs, often have strong locality properties, in the sense that they are more closely couple certain clusters of nearby coordinates than those farther apart with respect to the underlying topology. In such cases, modeling $A$ in terms of a basis of global eigenfunctions is both computationally wasteful and conceptually unreasonable: a localized dictionary would be more appropriate. In contrast to PCA, \cite{pmlr-v32-kondor14} proposed \textit{Multiresolution Matrix Factorization}, or MMF for short, to construct a sparse hierarchical system of $L$-level dictionaries. The corresponding matrix factorization is of the form
$$\mA = \mU_1^T \mU_2^T \dots \mU_L^T \mH \mU_L \dots \mU_2 \mU_1,$$
where $\mH$ is close to diagonal and $\mU_1, \dots, \mU_L$ are sparse orthogonal matrices with the following constraints:
\begin{enumerate}
\item Each $\mU_\ell$ is $k$-point rotation for some small $k$, meaning that it only rotates $k$ coordinates at a time. Formally, Def.~\ref{def:rotation-matrix} defines and Fig.~\ref{fig:rotation-matrix} shows an example of the $k$-point rotation matrix. 
\item There is a nested sequence of sets $\sS_L \subseteq \cdots \subseteq \sS_1 \subseteq \sS_0 = [n]$ such that the coordinates rotated by $\mU_\ell$ are a subset of $\sS_\ell$.
\item $\mH$ is an $\sS_L$-core-diagonal matrix that is formally defined in Def.~\ref{def:core-diagonal}.
\end{enumerate}

\begin{definition} \label{def:rotation-matrix}
We say that $\mU \in \mathbb{R}^{n \times n}$ is an \textbf{elementary rotation of order $k$} (also called as a $k$-point rotation) if it is an orthogonal matrix of the form
$$\mU = \mI_{n - k} \oplus_{(i_1, \cdots, i_k)} \mO$$
for some $\sI = \{i_1, \cdots, i_k\} \subseteq [n]$ and $\mO \in \sS\sO(k)$. We denote the set of all such matrices as $\sS\sO_k(n)$.
\end{definition}

The simplest case are second order rotations, or called Givens rotations, which are of the form
\begin{equation}
\mU = \mI_{n - 2} \oplus_{(i, j)} \mO = 
\begin{pmatrix}
\cdot & & & & \\
& \cos(\theta) & & -\sin(\theta) & \\
& & \cdot & & \\
& \sin(\theta) & & \cos(\theta) & \\
& & & & \cdot \\
\end{pmatrix},
\label{eq:givens}
\end{equation}
where the dots denote the identity that apart from rows/columns $i$ and $j$, and $\mO \in \sS\sO(2)$ is the rotation matrix of some angle $\theta \in [0, 2\pi)$. Indeed, Jacobi's algorithm for diagonalizing symmetric matrices \cite{Jacobi+1846+51+94} is a special case of MMF factorization over Givens rotations.

\begin{definition} \label{def:core-diagonal}
Given a set $\sS \subseteq [n]$, we say that a matrix $\mH \in \mathbb{R}^{n \times n}$ is $\sS$-core-diagonal if $\mH_{i, j} = 0$ unless $i, j \in \sS$ or $i = j$. Equivalently, $\mH$ is $\sS$-core-diagonal if it can be written in the form $\mH = \mD \oplus_{\sS} \overline{\mH}$, for some $\overline{H} \in \mathbb{R}^{\lvert \sS \rvert \times \lvert \sS \rvert}$ and $\mD$ is diagonal. We denote the set of all $\sS$-core-diagonal symmetric matrices of dimension $n$ as $\sH^{\sS}_n$.
\end{definition}

Here is an example of a \(\sS\)-core-diagonal matrix. Consider $n = 5$ and $\sS = \{2, 4\}$. A matrix $\mH \in \mathbb{R}^{5 \times 5}$ is $\sS$-core-diagonal if:

\[
\mH = \begin{pmatrix}
1 & 0 & 0 & 0 & 0 \\
0 & 2 & 0 & 3 & 0 \\
0 & 0 & 4 & 0 & 0 \\
0 & 3 & 0 & 5 & 0 \\
0 & 0 & 0 & 0 & 6
\end{pmatrix}
\]

This matrix can be decomposed as $\mH = \mD \oplus_{\sS} \overline{\mH}$, where:

\[
\mD = \begin{pmatrix}
1 & 0 & 0 \\
0 & 4 & 0 \\
0 & 0 & 6 \\
\end{pmatrix}, \quad
\overline{\mH} = \begin{pmatrix}
2 & 3 \\
3 & 5
\end{pmatrix}
\]


\subsection{Multiresolution analysis} \label{sec:mra}

\begin{definition} \label{def:mmf}
Given an appropriate subset $\sO$ of the group $\sS\sO(n)$ of $n$-dimensional rotation matrices, a depth parameter $L \in \mathbb{N}$, and a sequence of integers $n = d_0 \ge d_1 \ge d_2 \ge \dots \ge d_L \ge 1$, a \textbf{Multiresolution Matrix Factorization (MMF)} of a symmetric matrix $\mA \in \mathbb{R}^{n \times n}$ over $\sO$ is a factorization of the form
\begin{equation} \label{eq:mmf}
\mA = \mU_1^T \mU_2^T \dots \mU_L^T \mH \mU_L \dots \mU_2 \mU_1,
\end{equation}
where each $\mU_\ell \in \sO$ satisfies $[\mU_\ell]_{[n] \setminus \sS_{\ell - 1}, [n] \setminus \sS_{\ell - 1}} = \mI_{n - d_\ell}$ for some nested sequence of sets $\sS_L \subseteq \cdots \subseteq \sS_1 \subseteq \sS_0 = [n]$ with $\lvert \sS_\ell \rvert = d_\ell$, and $\mH \in \sH^{\sS_L}_n$ is an $\sS_L$-core-diagonal matrix.
\end{definition}

\begin{definition} \label{def:factorizable}
We say that a symmetric matrix $\mA \in \mathbb{R}^{n \times n}$ is \textbf{fully multiresolution factorizable} over $\sO \subset \sS\sO(n)$ with $(d_1, \dots, d_L)$ if it has a decomposition of the form described in Def.~\ref{def:mmf}.
\end{definition}

We formally define MMF in Defs.~\ref{def:mmf} and \ref{def:factorizable}. Furthermore, \cite{pmlr-v32-kondor14} has shown that MMF mirrors the classical theory of multiresolution analysis (MRA) on the real line \cite{192463} to discrete spaces. The functional analytic view of wavelets is provided by MRA, which, similarly to Fourier analysis, is a way of filtering some function space into a sequence of subspaces
\begin{equation}
\dots \subset \sV_{-1} \subset \sV_0 \subset \sV_1 \subset \sV_2 \subset \dots
\label{eq:subspace-sequence}
\end{equation}
However, it is best to conceptualize (\ref{eq:subspace-sequence}) as an iterative process of splitting each $\sV_\ell$ into the orthogonal sum $\sV_\ell = \sV_{\ell + 1} \oplus \sW_{\ell + 1}$ of a smoother part $\sV_{\ell + 1}$, called the \textit{approximation space}; and a rougher part $\sW_{\ell + 1}$, called the \textit{detail space} (see Fig.~\ref{fig:subspaces}). Each $\sV_\ell$ has an orthonormal basis $\Phi_\ell \triangleq \{\phi_m^\ell\}_m$ in which each $\phi$ is called a \textit{father} wavelet. Each complementary space $\sW_\ell$ is also spanned by an orthonormal basis $\Psi_\ell \triangleq \{\psi_m^\ell\}_m$ in which each $\psi$ is called a \textit{mother} wavelet. In MMF, each individual rotation $\mU_\ell: \sV_{\ell - 1} \rightarrow \sV_\ell \oplus \sW_\ell$ is a sparse basis transform that expresses $\Phi_\ell \cup \Psi_\ell$ in the previous basis $\Phi_{\ell - 1}$ such that:
$$\phi_m^\ell = \sum_{i = 1}^{\text{dim}(\sV_{\ell - 1})} [\mU_\ell]_{m, i} \phi_i^{\ell - 1},$$
$$\psi_m^\ell = \sum_{i = 1}^{\text{dim}(\sV_{\ell - 1})} [\mU_\ell]_{m + \text{dim}(\sV_{\ell - 1}), i} \phi_i^{\ell - 1},$$
in which $\Phi_0$ is the standard basis, i.e. $\phi_m^0 = e_m$; and $\text{dim}(\sV_\ell) = d_\ell = \lvert \sS_\ell \rvert$. In the $\Phi_1 \cup \Psi_1$ basis, $\mA$ compresses into $\mA_1 = \mU_1\mA\mU_1^T$. In the $\Phi_2 \cup \Psi_2 \cup \Psi_1$ basis, it becomes $\mA_2 = \mU_2\mU_1\mA\mU_1^T\mU_2^T$, and so on. Finally, in the $\Phi_L \cup \Psi_L \cup \dots \cup \Psi_1$ basis, it takes on the form $\mA_L = \mH = \mU_L \dots \mU_2\mU_1 \mA \mU_1^T\mU_2^T \dots \mU_L^T$ that consists of four distinct blocks (supposingly that we permute the rows/columns accordingly):
$$\mH = \begin{pmatrix} \mH_{\Phi, \Phi} & \mH_{\Phi, \Psi} \\ \mH_{\Psi, \Phi} & \mH_{\Psi, \Psi} \end{pmatrix},$$
where $\mH_{\Phi, \Phi} \in \mathbb{R}^{\text{dim}(\sV_L) \times \text{dim}(\sV_L)}$ is effectively $\mA$ compressed to $\sV_L$, $\mH_{\Phi, \Psi} = \mH_{\Psi, \Phi}^T = 0$ and $\mH_{\Psi, \Psi}$ is diagonal. MMF approximates $\mA$ in the form
$$\mA \approx \sum_{i, j = 1}^{d_L} h_{i, j} \phi_i^L {\phi_j^L}^T + \sum_{\ell = 1}^L \sum_{m = 1}^{d_\ell} c_m^\ell \psi_m^\ell {\psi_m^\ell}^T,$$
where $h_{i, j}$ coefficients are the entries of the $\mH_{\Phi, \Phi}$ block, and $c_m^\ell = \langle \psi_m^\ell, \mA \psi_m^\ell \rangle$ wavelet frequencies are the diagonal elements of the $\mH_{\Psi, \Psi}$ block.

In particular, the dictionary vectors corresponding to certain rows of $\mU_1$ are interpreted as level one wavelets, the dictionary vectors corresponding to certain rows of $\mU_2\mU_1$ are interpreted as level two wavelets, and so on (see Section \ref{sec:mra}). One thing that is immediately clear is that whereas Eq.~(\ref{eq:eigen}) diagonalizes $\mA$ in a single step, multiresolution analysis will involve a sequence of basis transforms $\mU_1, \mU_2, \dots, \mU_L$, transforming $\mA$ step by step as
\begin{equation}
\mA \rightarrow \mU_1\mA\mU_1^T \rightarrow \mU_2\mU_1\mA\mU_1^T\mU_2^T \rightarrow \dots \rightarrow \mU_L \dots \mU_2\mU_1\mA\mU_1^T\mU_2^T \dots \mU_L^T,
\label{eq:mmf-transform}
\end{equation}
so the corresponding matrix factorization must be a multilevel factorization
\begin{equation}
\mA \approx \mU_1^T \mU_2^T \dots \mU_\ell^T \mH \mU_\ell \dots \mU_2 \mU_1.
\label{eq:mmf-factorization}
\end{equation}
Fig.~\ref{fig:mmf-transform} depicts the multiresolution transform of MMF as in Eq.~(\ref{eq:mmf-transform}). Fig.~\ref{fig:mmf-factorization} illustrates the corresponding factorization as in Eq.~(\ref{eq:mmf-factorization}).

\begin{figure}[t]
$$
\xymatrix{
L_2(\sX) \ar[r] & \cdots \ar[r] & \sV_0 \ar[r] \ar[dr] & \sV_1 \ar[r] \ar[dr] & \sV_2 \ar[r] \ar[dr] & \cdots \\
& & & \sW_1 & \sW_2 & \sW_3
}
$$
\caption{\label{fig:subspaces}
Multiresolution analysis splits each function space $\sV_0, \sV_1, \dots$ into the direct sum of a smoother part $\sV_{\ell + 1}$ and a rougher part $\sW_{\ell + 1}$.
}
\end{figure}
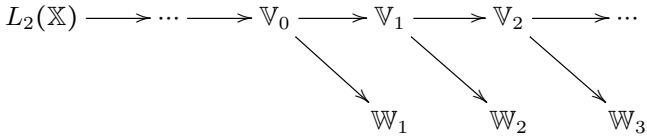
\newcommand{\tikmxA}[1]{
\bigg(\,\begin{tikzpicture}[baseline=-17, scale=0.06]
\filldraw[gray] (0,0) rectangle +(#1,-#1); \end{tikzpicture}\,\bigg)}

\newcommand{\tikmxB}[2]{
\bigg(\,\begin{tikzpicture}[baseline=-17, scale=0.06]
\filldraw[gray] (0,0) rectangle +(#1,-#1);
\foreach \i in {#1,...,#2}{
\filldraw[gray] (\i,-\i) rectangle +(1,-1);}
\end{tikzpicture}\,\bigg)}

\newcommand{\tikmxC}[2]{
\bigg(\,\begin{tikzpicture}[baseline=-17, scale=0.06]
\draw (0,0) rectangle +(17,-17);\draw (10,-9) node {#1}; \end{tikzpicture} \,\bigg)}

\newcommand{\tikmxD}[3]{
\bigg(\,\begin{tikzpicture}[baseline=-17, scale=0.06]
\filldraw[gray] (#1,-#1) rectangle +(#2,-#2);
\foreach \i in {0,...,#3}{
\filldraw[gray] (\i,-\i) rectangle +(1,-1);}
\end{tikzpicture}\,\bigg)}

\begin{figure}[t]
\[
\mI_{n - k} \oplus_{(i_1, .., i_k)} \mO = 
\,\Pi\, \underset{\displaystyle \mU}{\tikmxD{7}{4}{17}} \,\Pi^\top
\]\\ \vspace{-10pt}\mbox{}
\caption{\label{fig:rotation-matrix}
A rotation matrix of order $k$. The purpose of permutation matrix $\Pi$ is solely to ensure that the blocks of the matrices appear contiguous in the figure. In this case, $n = 17$ and $k = 4$.
}
\end{figure}

\begin{figure}[t]
\[
\,\Pi\, \underset{\displaystyle \mA}{\tikmxA{17}} \,\Pi^\top \xrightarrow{~U_1~} 
\underset{\displaystyle \mA_1 = \mU_1 \mA \mU_1^T}{\tikmxB{14}{17}}\xrightarrow{~U_2~}
\underset{\displaystyle \mA_2 = \mU_2 \mA_1 \mU_2^T}{\tikmxB{10}{17}}\xrightarrow{~~}
\ldots\xrightarrow{~~}
\underset{\displaystyle \mA_L = \mH}{\tikmxB{7}{17}}
\]\\ \vspace{-10pt}\mbox{}
\caption{\label{fig:mmf-transform}
MMF can be thought of as a process of successively compressing $\mA$ to size $d_1 \times d_1$, $d_2 \times d_2$, etc. (plus the diagonal entries) down to the final $d_L \times d_L$ core-diagonal matrix $\mH$ (see Def.~\ref{def:mmf}). The role of permutation matrix $\Pi$ is purely for the ease of visualization (as in Fig.~\ref{fig:rotation-matrix}).
}
\end{figure}

\begin{figure}[t]
\[
\,\Pi\, \underset{\displaystyle A}{\tikmxA{17}} \,\Pi^\top 
\approx
\underset{\displaystyle U_1^T}{\tikmxD{7}{4}{17}}
\ldots
\underset{\displaystyle U_L^T}{\tikmxD{14}{4}{17}}
\underset{\displaystyle H}{\tikmxD{0}{7}{17}}
\underset{\displaystyle U_L}{\tikmxD{14}{4}{17}}
\ldots
\underset{\displaystyle U_1}{\tikmxD{7}{4}{17}}
\]\\ \vspace{-20pt}\mbox{}
\caption{\label{fig:mmf-factorization}
Matrix approximation as in Eq.~\ref{eq:mmf}. In this figure, the core block size of each rotation matrix $\mU_\ell$ and $\mH$ are $k \times k = 4 \times 4$ and $d_L \times d_L = 8 \times 8$, respectively. Permutation matrix $\Pi$ is only for visualization (as in Figs.~\ref{fig:rotation-matrix}~\ref{fig:mmf-transform}).
}
\end{figure}

\subsection{Optimization by heuristics} \label{sec:Optimization}

Heuristically, factorizing $\mA$ can be approximated by an iterative process that starts by setting $\mA_0 = \mA$ and $\sS_1 = [n]$, and then executes the following steps for each resolution level $\ell \in \{1, \dots, L\}$:
\begin{enumerate}
\item Given $\mA_{\ell - 1}$, select $k$ indices $\sI_\ell = \{i_1, \dots, i_k\} \subset \sS_{\ell - 1}$ of rows/columns of the active submatrix $[\mA_{\ell - 1}]_{\sS_{\ell - 1}, \sS_{\ell - 1}}$ that are highly correlated with each other.
\item Find the corresponding $k$-point rotation $\mU_\ell$ to $\sI_\ell$, and compute $\mA_\ell = \mU_\ell \mA_{\ell - 1} \mU_\ell^T$ that brings the submatrix $[\mA_{\ell - 1}]_{\sI_\ell, \sI_\ell}$ close to diagonal. In the last level, we set $\mH = \mA_L$ (see Fig.~\ref{fig:mmf-transform}).
\item Determine the set of coordinates $\sT_\ell \subseteq \sS_{\ell - 1}$ that are to be designated wavelets at this level, and eliminate them from the active set by setting $\sS_\ell = \sS_{\ell - 1} \setminus \sT_\ell$.
\end{enumerate}

\section{Stiefel Manifold Optimization} \label{sec:proof}

In order to solve the MMF optimization problem, we consider the following generic optimization with orthogonality constraints:
\begin{equation}
\min_{\mX \in \mathbb{R}^{n \times p}} \mathcal{F}(\mX), \ \ \text{s.t.} \ \ \mX^T \mX = \mI_p,
\end{equation}

We identify tangent vectors to the manifold with $n \times p$ matrices. We denote the tangent space at $\mX$ as $\mathcal{T}_{\mX} \mathcal{V}_p(\mathbb{R}^n)$. Lemma \ref{lemma:tangent} characterizes vectors in the tangent space.

\begin{lemma}
Any $\mZ \in \mathcal{T}_{\mX} \mathcal{V}_p(\mathbb{R}^n)$, then $\mZ$ (as an element of $\mathbb{R}^{n \times p}$) satisfies
$$\mZ^T \mX + \mX^T \mZ = 0,$$
where $\mZ^T \mX$ is a skew-symmetric $p \times p$ matrix.
\label{lemma:tangent}
\end{lemma}

\begin{proof}
Let $\mY(t)$ be a curve in $\mathcal{V}_p(\mathbb{R}^n)$ that starts from $\mX$. We have:
\begin{equation}
\mY^T(t)\mY(t) = \mI_p.
\label{eq:identity}
\end{equation}
We differentiate two sides of Eq.~(\ref{eq:identity}) with respect to $t$:
$$\frac{d}{dt}(\mY^T(t)\mY(t)) = 0$$
that leads to:
$$\bigg(\frac{d\mY}{dt}(0)\bigg)^T \mY(0) + \mY(0)^T \frac{d\mY}{dt}(0) = 0$$
at $t = 0$. Recall that by definition, $\mY(0) = \mX$ and $\frac{d\mY}{dt}(0)$ is any element of the tangent space at $\mX$. Therefore, we arrive at $\mZ^T\mX + \mX^T\mZ = 0$.
\end{proof}

Suppose that $\mathcal{F}$ is a differentiable function. The gradient of $\mathcal{F}$ with respect to $\mX$ is denoted by $\mG \triangleq \mathcal{D}\mathcal{F}_{\mX} \triangleq \big(\frac{\partial \mathcal{F}(\mX)}{\partial \mX_{i, j}}\big)$. The derivative of $\mathcal{F}$ at $\mX$ in a direction $\mZ$ is
$$\mathcal{D}\mathcal{F}_{\mX}(\mZ) \triangleq \lim_{t \rightarrow 0} \frac{\mathcal{F}(\mX + t\mZ) - \mathcal{F}(\mX)}{t} = \langle \mG, \mZ \rangle$$

Since the matrix $\mX^T \mX$ is symmetric, the Lagrangian multiplier $\Lambda$ corresponding to $\mX^T\mX = \mI_p$ is a symmetric matrix. The Lagrangian function of problem (\ref{eq:opt-prob}) is
\begin{equation}
\mathcal{L}(\mX, \mLambda) = \mathcal{F}(\mX) - \frac{1}{2} \text{trace}(\mLambda (\mX^T \mX - \mI_p))
\label{eq:lagrangian}
\end{equation}

\begin{lemma}
Suppose that $\mX$ is a local minimizer of problem (\ref{eq:opt-prob}). Then $\mX$ satisfies the first-order optimality conditions $\mathcal{D}_{\mX}\mathcal{L}(\mX, \mLambda) = \mG - \mX\mG^T\mX = 0$ and $\mX^T\mX = \mI_p$ with the associated Lagrangian multiplier $\mLambda = \mG^T\mX$. Define $\displaystyle \nabla \mathcal{F}(\mX) \triangleq \mG - \mX \mG^T \mX$ and $\mA \triangleq \mG \mX^T - \mX \mG^T$. Then $\displaystyle \nabla \mathcal{F} = \mA \mX$. Moreover, $\displaystyle \nabla \mathcal{F} = 0$ if and only if $\mA = 0$.
\label{lemma:first-order-condition}
\end{lemma}

\begin{proof}
Since $\mX \in \mathcal{V}_p(\mathbb{R}^n)$, we have $\mX^T\mX = \mI_p$. We differentiate both sides of the Lagrangian function:
$$\mathcal{D}_\mX \mathcal{L}(\mX, \mLambda) = \mathcal{D}\mathcal{F}(\mX) - \mX \mLambda = 0.$$
Recall that by definition, $\mG \triangleq \mathcal{D}\mathcal{F}(\mX)$, we have
\begin{equation}
\mathcal{D}_\mX \mathcal{L}(\mX, \mLambda) = \mG - \mX \mLambda = 0.
\end{equation}
Multiplying both sides by $\mX^T$, we get $\mX^T\mG - \mX^T\mX \mLambda = 0$ that leads to $\mX^T\mG - \mLambda = 0$ or $\mLambda = \mX^T\mG$. Since the matrix $\mX^T\mX$ is symmetric, the Lagrangian multiplier $\mLambda$ correspoding to $\mX^T\mX = \mI_p$ is a symmetric matrix. Therefore, we obtain $\mLambda = \mLambda^T = \mG^T\mX$ and $\mathcal{D}_\mX \mathcal{L}(\mX, \mLambda) = \mG - \mX\mG^T\mX = 0$. By definition, $\mA \triangleq \mG\mX^T - \mX\mG^T$. We have $\mA\mX = \mG - \mX\mG^T\mX = \nabla \mathcal{F}$. The last statement is trivial.
\end{proof}

Let $\mX \in \mathcal{V}_p(\mathbb{R}^n)$, and $\mW$ be any $n \times n$ skew-symmetric matrix. We consider the following curve that transforms $\mX$ by $\big( \mI + \frac{\tau}{2} \mW \big)^{-1} \big(\mI - \frac{\tau}{2} \mW\big)$:
\begin{equation}
\mY(\tau) = \big( \mI + \frac{\tau}{2} \mW \big)^{-1} \big(\mI - \frac{\tau}{2} \mW\big) \mX.
\label{eq:cayley}
\end{equation}
This is called as the \textit{Cayley transformation}. Its derivative with respect to $\tau$ is
\begin{equation}
\mY'(\tau) = -\bigg(\mI + \frac{\tau}{2} \mW\bigg)^{-1} \mW \bigg( \frac{\mX + \mY(\tau)}{2} \bigg).
\label{eq:y-derivative}
\end{equation}
The curve has the following properties:
\begin{enumerate}
\item It stays in the Stiefel manifold, i.e. $\mY(\tau)^T \mY(\tau) = \mI$.
\item Its tangent vector at $\tau = 0$ is $\mY'(0) = -\mW \mX$. It can be easily derived from Lemma \ref{lemma:tangent} that $\mY'(0)$ is in the tangent space $\mathcal{T}_{\mY(0)} \mathcal{V}_p(\mathbb{R}^n)$. Since $\mY(0) = X$ and $\mW$ is a skew-symmetric matrix, by letting $\mZ = -\mW\mX$, it is trivial that $\mZ^T\mX + \mX^T\mZ = 0$.
\end{enumerate}

\begin{lemma}
If we set $\mW \triangleq \mA \triangleq \mG\mX^T - \mX\mG^T$ (see Lemma ~\ref{lemma:first-order-condition}), then the curve $\mY(\tau)$ (defined in Eq.~(\ref{eq:cayley})) is a decent curve for $\mathcal{F}$ at $\tau = 0$, that is
$$\mathcal{F}'_\tau(\mY(0)) \triangleq \frac{\partial \mathcal{F}(\mY(\tau))}{\partial \tau}\bigg\vert_{\tau = 0} = -\frac{1}{2} \|\mA\|_F^2.$$
\label{lemma:descent-curve}
\end{lemma}

\begin{proof}
By the chain rule, we get 
$$\mathcal{F}'_\tau(\mY(\tau)) = \text{trace}(\mathcal{D}\mathcal{F}(\mY(\tau))^T \mY'(\tau)).$$
At $\tau = 0$, $\mathcal{D}\mathcal{F}(\mY(0)) = \mG$ and $\mY'(0) = -\mA\mX$. Therefore, 
$$\mathcal{F}'_\tau(\mY(0)) = -\text{trace}(\mG^T(\mG\mX^T - \mX\mG^T)\mX) = -\frac{1}{2}\text{trace}(\mA\mA^T) = -\frac{1}{2}\|\mA\|_F^2.$$
\end{proof}

It is well known that the steepest descent method with a fixed step size may not converge, but the convergence can be guaranteed by choosing the step size wisely: one can choose a step size by minimizing $\mathcal{F}(\mY(\tau))$ along the curve $\mY(\tau)$ with respect to $\tau$ \cite{Wen10}. With the choice of $\mW$ given by Lemma \ref{lemma:descent-curve}, the minimization algorithm using $\mY(\tau)$ is roughly sketched as follows: Start with some initial $\mX^{(0)}$. For $t > 0$, we generate $\mX^{(t + 1)}$ from $\mX^{(t)}$ by a curvilinear search along the curve $\mY(\tau) = \big( \mI + \frac{\tau}{2} \mW \big)^{-1} \big(\mI - \frac{\tau}{2} \mW\big) \mX^{(t)}$ by changing $\tau$. Because finding the global minimizer is computationally infeasible, the search terminates when then Armijo-Wolfe conditions that indicate an approximate minimizer are satisfied. The Armijo-Wolfe conditions require two parameters $0 < \rho_1 < \rho_2 < 1$ \cite{NoceWrig06} \cite{Wen10} \cite{Tagare2011NotesOO}:
\begin{equation}
\mathcal{F}(\mY(\tau)) \leq \mathcal{F}(\mY(0)) + \rho_1\tau\mathcal{F}'_\tau(\mY(0))
\label{eq:condition-1}
\end{equation}
\begin{equation}
\mathcal{F}'_\tau(\mY(\tau)) \ge \rho_2\mathcal{F}'_\tau(\mY(0))
\label{eq:condition-2}
\end{equation}
where $\mathcal{F}'_\tau(\mY(\tau)) = \text{trace}(\mG^T\mY'(\tau))$ while $\mY'(\tau)$ is computed as Eq.~(\ref{eq:y-derivative}) and $\mY'(0) = -\mA\mX$. The gradient descent algorithm on Stiefel manifold to optimize the generic orthogonal-constraint problem (\ref{eq:opt-prob}) with the curvilinear search submodule is described in Algorithm \ref{alg:stiefel}, which is used as a submodule in part of our learning algorithm to solve the MMF in (\ref{eq:mmf-opt}). The algorithm can be trivially extended to solve problems with multiple variables and constraints.

\begin{algorithm}
\caption{Stiefel manifold gradient descent algorithm} \label{alg:stiefel}
\begin{algorithmic}[1]
\State Given $0 < \rho_1 < \rho_2 < 1$ and $\epsilon > 0$.
\State Given an initial point $\mX^{(0)} \in \mathcal{V}_p(\mathbb{R}^n)$.
\State $t \gets 0$
\While{true}
	\State $\mG \gets \big(\frac{\partial \mathcal{F}(\mX^{(t)})}{\partial \mX^{(t)}_{i, j}}\big)$ \Comment{Compute the gradient of $\mathcal{F}$ w.r.t $\mX$ elemense-wise}
	\State $\mA \gets \mG{\mX^{(t)}}^T - \mX^{(t)}\mG^T$ \Comment{See Lemma 2, 3}
	\State Initialize $\tau$ to a non-zero value. \Comment{Curvilinear search for the optimal step size}
	\While{(\ref{eq:condition-1}) and (\ref{eq:condition-2}) are \textbf{not} satisfied} \Comment{Armijo-Wolfe conditions}
		\State $\tau \gets \frac{\tau}{2}$ \Comment{Reduce the step size by half}
	\EndWhile
	\State $\mX^{(t + 1)} \gets \mY(\tau)$ \Comment{Update by the Cayley transformation}
	\If{$\|\nabla \mathcal{F}(\mX^{(t + 1)})\| \leq \epsilon$} \Comment{Stopping check. See Lemma 2.}
		\State \textbf{STOP}
	\Else
		\State $t \gets t + 1$
	\EndIf
\EndWhile
\end{algorithmic}
\end{algorithm}
\end{appendices}

\end{document}